\documentclass[11pt]{article}

\usepackage{amsmath,amssymb,amsthm}
\usepackage{graphicx}
\usepackage{booktabs}
\usepackage{multirow}
\usepackage{subcaption}
\usepackage{natbib}
\usepackage{algorithm}
\usepackage{algpseudocode}
\usepackage{geometry}
\usepackage{xcolor}
\usepackage{hyperref}
\hypersetup{
  colorlinks=true,
  linkcolor=blue,
  citecolor=blue,
  urlcolor=blue
}
\usepackage{dsfont}
\usepackage{mathtools}
\usepackage{float}
\usepackage[symbol]{footmisc}

\newtheorem{assumption}{Assumption}
\newtheorem{property}{Property}
\newtheorem{theorem}{Theorem}
\newtheorem{lemma}{Lemma}
\newtheorem{corollary}{Corollary}
\newtheorem{definition}{Definition}

\newcommand{\Input}{\Require}

\title{Nonlinear Bandit}
\date{}

\begin{document}

\begin{center}
    {\Huge Nonlinear Bandit} \\[1.5em]
    {\large Tianshuo Zheng\textsuperscript{1}, Ting Wu\textsuperscript{1},  Zhi-Hua Zhou\textsuperscript{2} and Keqin Liu\footnote{Corresponding author: keqin.liu@xjtlu.edu.cn}\textsuperscript{3}} \\[1em]
    \textsuperscript{1}School of Mathematics, Nanjing University, Nanjing, 210093, China \\
  \textsuperscript{2} School of Artificial Intelligence, Nanjing University, National Key Laboratory for Novel Software Technology, Nanjing, 210023, China \\
  \textsuperscript{3} School of Mathematics and Physics, Xi'an Jiaotong-Liverpool University, Suzhou, 215123, China
\end{center}

\begin{abstract}
In this paper we first study the problem of generalized linear bandit (GLB) under heavy-tailed noise. The characteristics of heavy-tailed distributions are widely observed in real-world applications such as personalized recommendation, financial markets, and medical treatments. Based on the online mirror descent (OMD) method, we propose an algorithm EHM that extends the adaptive Huber loss method \citep{wang2025heavy} with one-pass update ($\mathcal{O}(1)$ computational complexity with respect to current round $t$ and the time horizon $T$), which simultaneously achieves an almost optimal regret of $\widetilde{\mathcal{O}}( T^{\frac{1}{1+\epsilon}})$ where $T$ is the time horizon. In addition, by utilizing a special property of some link function  \citep{sawarni2025generalized}, our algorithm eliminates the need to know a commonly used parameter. Next, we study the GLB problem under the case when contextual characteristic becomes piecewise constant, and we slightly revised former algorithm to obtain the PGLB-EHM algorithm. After theoretical analysis, we prove that the regret upper bound order stays the same. Furthermore, we look deeper into a special case of nonlinear bandit (NB) and present the NB-EHM algorithm with bisection method and special restriction. Eventually we utilize the affine lifting approach and show that the general NB problem can be applied with NB-EHM to achieve a sublinear regret bound.
\end{abstract}

\section{Introduction}
Online sequential decision-making under uncertainty has long been a  central challenge in reinforcement learning, operations research and stochastic optimization. How to balance between exploration and exploitation to maximize the long-term gained reward becomes the main question. The multi-armed bandit (MAB) has served as a canonical framework for studying sequential learning problems. By selecting an action in each round and receiving a random reward, the player aims to minimize the growth rate of regret (cost of learning), thereby achieving a higher reward gaining rate over the long run\citep{lai1985asymptotically}. However, traditional MAB models assume a finite action set and ignore the underlying relationships between actions and the environmental context. Therefore, it limits bandit theory's applicability in real-world scenarios such as online advertising and recommendation algorithms. The contextual bandit (CB) paradigm addresses this limitation by incorporating contextual features into the decision-making process. Within this paradigm, the linear bandit (LB) assumes that the expected reward depends linearly on multi-dimensional features, typically represented as the inner product of an action and a context vector. The modified model underpins classical algorithms such as LinUCB \citep{li2010contextual} and LinTS \citep{agrawal2013thompson}, which perform well in low-dimensional settings with light-tailed reward. However, many practical tasks, such as predicting user click-through rates, page views, and repeat purchase counts, exhibit highly non-linear relationships and complex, non-Gaussian reward distributions. Under such circumstances, simple linear models often fail to provide accurate predictions.

To enhance the expressive power, generalized linear model (GLM) introduces a link function that maps linear predictions to a sample space better aligned with the data distribution, naturally accommodating distributions such as logistic, Poisson, and Pareto \citep{mccullagh1989generalized}. \cite{filippi2010parametric} were the first to integrate GLM into the bandit framework, proposing the GLM-UCB algorithm, while \cite{li2017provably} provides a theoretical foundation for the generalized linear bandit (GLB) by establishing optimal regret bounds under high-dimensional features. Nevertheless, most existing GLB algorithms rely on maximum likelihood estimation (MLE) and are sensitive to noisy or corrupted environments, where parameter estimates can become biased, leading to degraded decision-making performance \citep{li2024variance}. Approaches such as truncation and mean of medians have been employed to achieve sublinear regret under sub-Gaussian or heavy-tailed noise. Originally introduced in linear bandit, the adaptive Huber loss combines the smoothness of the $L_2$ loss with the robustness of the $L_1$ loss, offering tolerance to outliers and effectively making up for the bias of MLE under heavy-tailed distributions. By extending Huber loss to the GLB framework, our proposed algorithm maintains robustness while significantly reducing computational complexity. Therefore, studying Huber loss-based GLB algorithms is crucial for advancing online learning theory and holds significant potential for practical applications in extremely noisy environments \citep{chapelle2011empirical,bastani2020online}.

Nonlinear bandit (NB) extends GLB to common nonlinear reward structures beyond link functions. Lattimore first considers a simple setting where bandit model is quadratic with bounded reward and no noise \citep{lattimore2026bandit}. By assuming reward function $m:\mathcal{U} \rightarrow \mathbb{R}$ to be uniformly locally $\alpha$-Holder continuous,~\cite{agrawal1995} achieves $\mathcal{O}(T^{(2\alpha+1)/(3\alpha+1)+\eta})$ regret under light-tailed noise, where $\mathcal{U}$ is a compact subset in $\mathbb{R}$, $\alpha \in (0,1]$ and arbitary $\eta >0$. Kernel-based methods assume that the reward function lies in a reproducing kernel Hilbert space (RKHS) with sub-Gaussian (light-tailed) noise. Typical regret is $\mathcal{O}(\sqrt{T\gamma_T})$,
where \(\gamma_T\) is the information gain. Despite being nonparametric, these methods suffer from strong RKHS realizability assumptions and kernel misspecification sensitivity. They all lead to high computational cost $\mathcal{O}(T^2) \sim \mathcal{O}(T^3)$ in total, therefore constraining algorithm's scalability \citep{filippi2010parametric,valko2013finite}. Neural bandits leverage neural tangent kernel (NTK) approximations to analyze overparameterized networks. Under infinite-width and sub-Gaussian noise assumptions, regret matches kernel-type bounds up to network-dependent factors~\citep{lisicki2021empirical}: $\mathcal{O}(\sqrt{T\gamma_T})$.
However, performance relies on the NTK regime, which may not hold in finite-width networks, and training complexity remains high due to repeated optimizations~\citep{xu2024stochastic}. Recursive weighted Gaussian process (RWGP) methods \citep{wang2024online} extended GP-UCB to non-stationary settings via exponential forgetting. By assuming \(f_t \in \mathcal{H}_k\) with bounded RKHS norm and sub-Gaussian noise, 
the regret satisfies $\mathcal{R}_T = \mathcal{O}(\sqrt{T\gamma_T}) + \mathcal{O}(V_T T^\beta)$. Here the rate exponent $\beta$ depends on the forgetting kernel analysis. While adaptive to temporal drift, RWGP suffers from  sensitivity to forgetting factors and \(\mathcal{O}(T^2)\) kernel computational complexity. Bandit convex optimization (BCO) and zeroth-order methods approximate gradients using random perturbations under bandit feedback. Under convexity, Lipschitz smoothness, and sub-Gaussian noise model, classical results \citep{flaxman2005online} achieved regret $\mathcal{O}(T^{3/4})$, improvable to \(\mathcal{O}(\sqrt{T})\) under strong convexity~\citep{hazan2016introduction}. \cite{fokkema2026xai_bandit} further connected convex bandits with explainable AI, showing that convexity enables both tractable online optimization and interpretable decision boundaries. While in nonconvex settings, bandit feedback typically reduces to minimize gradient norm rather than regret. Zeroth-order stochastic methods \citep{liu2018zeroth} achieved $\min_{t \le T} \mathbb{E}\|\nabla f(x_t)\|^2 = \mathcal{O}(T^{-1/2})$, which does not imply a sublinear regret in the bandit context.

Balancing statistical efficiency and computational complexity while maintaining  sublinear regret remains a core challenge. Moreover, current methods rarely connect GLB and NB together to give efficient algorithms, which is a valuable way that we will leverage in this paper. Therefore we study the general nonlinear bandit under heavy-tailed noise without assuming GLM structure, RKHS realizability, NTK regime, or convexity. Existing methods mainly rely on these restrictive structural assumptions and study bounded or sub-Gaussian noise. Our approach achieves a sublinear regret while maintaining robustness under heavy-tailed reward distributions and $\mathcal{O}(1)$ computational complexity with respect to current round $t$ and the time horizon $T$.

\subsection{Main Contributions}
We introduce the GLB-EHM algorithm for the heavy-tailed GLB problem, which achieves the robustness by leveraging the Huber loss \citep{huber1964robust}. Our regret uper bound matches the lower bound $\Omega(dT^{\frac{1}{1+\varepsilon}})$ in \citep{shao2018almost} up to logarithmic factors. Additionally EHM removes the dependency on $\kappa$ with a special assumption compared to \citep{xue2023efficient}, while Lipschitz constant~$L$ is also eliminated from the leading term in contrast with \cite{yu2025corruption}. 

Next, we offer the PGLB-EHM algorithm under a specific nonlinear scenario when contextual characteristic is piecewise constant (PGLB). Theoretical analysis shows that it remains the same regret order as GLB-EHM by distinguishing the optimal area with expected time cost $\leq \widetilde{\mathcal{O}}(T^{\frac{1}{1+\varepsilon}})$. Subsequently we propose the NB-EHM algorithm that focuses on a smooth nonlinear bandit class with a special type of function structure (similar to PGLB), and adopts the bisection method to achieve a sublinear regret using some effective exploration strategies. To be explicit, as division-by-segment goes on, time limit on refined areas will increase with respect to their widths. Meanwhile estimation error inside each zone becomes smaller automatically. This indicates when the algorithm keeps exploring regions with smaller widths, it appears more likely to approach the area corresponding to the optimal reward. In the end, we leverage the affine lifting approach to explain that NB-EHM is compatible with general NB while still achieving a sublinear regret. We have presented numerical simulations for every algorithm and they all produced desired results, offering a new idea for further research in heavy-tailed GLB and NB problems (see Section~\ref{SNB-EHM} for more details).

\subsection{Related Work}
Early research on MAB problems primarily focused on the exploration-exploitation trade-off within a pure strategy space. Since \cite{thompson1933likelihood} introduced the heuristic idea of posterior sampling, the classical bandit has provided a fresh perspective for studying sample distributions. Subsequently, \cite{lai1985asymptotically} established asymptotic optimal regret bound theory, leading to the fundamental principle of "optimism in the face of uncertainty" (OFU). Building on this, \cite{auer2002finite} proposed the UCB1 algorithm, providing the first explicit finite-time regret bounds. Later, \cite{cappe2013kullback} conducted a deeper analysis of Thompson sampling, demonstrating that its performance matches UCB algorithms for the same growth order. \cite{li2010contextual} applied LinUCB to news recommendation, empirically highlighting the critical role of contextual information in sequential decision-making. Since then, linear contextual bandits have been extensively studied. \cite{chu2011contextual} further refined the theoretical regret analysis of LinUCB. \cite{agrawal2013thompson} introduced the LinTS algorithm from a frequentist perspective, thereby broadening the applicability of linear models. However, despite the minimax lower bounds have been established by \cite{dani2008stochastic} and \cite{rusmevichientong2010linearly}, linear methods are inherently constrained by their limited expressive power, a limitation that generalized linear models (GLMs) help to overcome. \cite{filippi2010parametric} were the first to integrate GLMs into contextual bandits, proposing the GLM-UCB algorithm which achieves near-optimal regret bounds under a maximum likelihood estimation (MLE) framework. Subsequently, \cite{li2017provably} introduced UCB-GLM with adaptive step sizes, precisely balancing the  relationship between observed rewards and parameter updates. Meanwhile, \cite{jun2017scalable} and \cite{kveton2018thompson} developed Thompson sampling variants in GLB problem over large-scale parallel settings. Nevertheless, when it comes to heavy-tailed noise case, original methods fail to maintain robustness in distributional estimation \citep{li2024variance}.

Over decades, study over heavy-tailed reward has been developed in bandit theory. RobustUCB laid foundation for heavy-tailed MAB research and proved the minimax bound of regret is $\mathcal{O}(T^{\frac{1}{1+\varepsilon}})$ when $(1+\varepsilon)$-th moment of noise is bounded \citep{bubeck2013bandits}. Later KLinf-UCB utilized KL divergence to give implicit mean estimation, first achieving optimal regret bound under heavy tail distributions \citep{agrawal2021regret}. This year ERUCB of \cite{liu2026extended} successfully extended assumption of bounded scale-free kurtosis by \cite{lattimore2017scale} and achieved an $\mathcal{O}(\log T)$ regret bound through introducing a moment control coefficient. While in heavy-tailed GLB problem, \cite{xue2023efficient} utilized truncation (CRTM) and mean-of-median (CRMM) to achieve near optimal regret with relatively high computational efficiency. \cite{yu2025corruption} studied Huber regression under adversarial attacks and noise with finite variances. \cite{sawarni2025generalized} removed the $\kappa$ dependency in the leading coefficient, by assuming specific reward distributions. \cite{zhong2021breaking} relaxed the moment condition to super heavy-tailed case ($\varepsilon=0$) and achieved regret bound of $\widetilde{\mathcal{O}}(\sqrt{d}T^{\frac{2}{3}})$ under symmetric distributions. Finally, recent years have witnessed the successful extensions of GLB research to broader applications, including cascading ranking \citep{cai2022cascading}, information-directed sampling \citep{russo2016information}, and deep models supported by neural ordinary differential equations \citep{chen2018neural}.

\section{Preliminaries}\label{section_preliminaries}
In this section, we briefly introduce the basic model of GLB and provide the necessary assumptions. We denote $\left\lVert \cdot \right\rVert_p$ as the p-norm. For a positive semi-definite $M\in \mathbb{R}^{d\times d}$, $\left\lVert x\right\rVert_M \triangleq \sqrt{x^TMx}$ where $x\in \mathbb{R}^d$. Let $ B_d \triangleq \{ x \in \mathbb{R}^d | \left\lVert x \right\rVert_{2}\leq 1\}$ refers to the closed unit ball. $\langle\cdot,\cdot\rangle$ represents common inner product in $\mathbb{R}^d$.

\subsection{Generalized Linear Bandit with Heavy-tailed Noise}\label{GLB with hvt noise}
The GLB problem involves sequential decision-making within $T$ steps. Let $\{\mathcal{F}_t\}_{t\geq1}$ refers to the filtration where $\mathcal{F}_{t-1} \triangleq \sigma\{X_1,r_1,X_2,..,X_{t-1},r_{t-1},X_t\}$. At each time round $t \in  [T] \triangleq \{1,2,..,T\} $, the player is supposed to select an arm $X_t$ from the given feasible set, denoted as $\mathcal{X}_t$. Consequently a reward $r_t$ will be observed and its expression is given as follows:
\begin{align*}
    r_t=\mu(\langle X_t, \theta_*\rangle)+\varepsilon_t,
\end{align*}
where $\theta_*\in \Theta \subset \mathbb{R}^d$ is an unknown  parameter and $\varepsilon_t$ is the random noise following some heavy-tailed distribution. The noise is assumed to satisfy the following conditions: $\mathbb{E}[\varepsilon_t|\mathcal{F}_{t-1}]=0$ and $\mathbb{E}[|\varepsilon_t|^{1+\epsilon}|\mathcal{F}_{t-1}]=\nu_t^{1+\epsilon}$ for some $\varepsilon \in (0,1]$. Therefore, the expected reward at time~$t$ can be written in the form:
\begin{align*}
    \mathbb{E}[r_t|\mathcal{F}_{t-1}]=\mu(\langle X_t, \theta_*\rangle).
\end{align*}

The learner aims to maximize the expected total reward over the time horizon $T$, which is equivalent to minimizing the cumulative regret:
\begin{align*}
    \mathcal{R}(T)=\sum_{t=1}^{T} [\mu(\langle X_t^*, \theta_*\rangle)-\mu(\langle X_t, \theta_*\rangle)],
\end{align*}
where $X_t^*= \mathop{\arg\max}_{X \in \mathcal{X}_t} \mu(\langle X, \theta_*\rangle)$ stands for the optimal arm at time $t$. In addition, several assumptions concerning the feasible set $\mathcal{X}_t$, unknown parameter $\theta_* \in \Theta$ and link function $\mu(\cdot)$.

\begin{assumption}\label{assumption_bd}
The feasible set and unknown parameter are both bounded: $\forall t \in [T], \forall X \in\mathcal{X}_{t},X \in B_d$ and $\Theta=\{\theta \in \mathbb{R}^d| \left\lVert \theta \right\rVert_{2} \leq S\}$ , where $S>0$ is a known constant.
\end{assumption}

\begin{assumption}\label{assumption_linkfunc}
The link function $\mu$ is $L$-Lipschitz continuous in $[-S,S]$, is twice differentiable over $[-S,S]$ and $\exists \kappa>0$ such that $\mu'(\cdot) \geq \kappa$ in $[-S,S]$. Additionally, there exists a given constant $K>0$ satisfying $\mu''(\cdot) \leq K \mu'(\cdot)$ in $[-S,S]$ (same as in \cite{sawarni2025generalized}). 
\end{assumption}

\subsection{Huber Loss}

By denoting $z_t(\theta) \triangleq \frac{r_t-\mu(\langle X_t, \theta\rangle)}{\sigma_t}$. We introduce an extended version of the Huber loss for the GLB problem, which is used in parameter calculation in Algorithm \ref{alg:EHM} and in the proof of Theorem \ref{thm:confidence_region_bound} below:
\begin{align} \label{extended_huber_loss}
\ell_t(\theta) :=
\begin{cases}
\frac{-r_t\langle X_t, \theta\rangle+\int_{0}^{\langle X_t, \theta\rangle} \mu(x)dx}{\sigma_t^2} & \text{if } |{z}_{t}(\theta)| \leq \tau_{t}, \\
-\tau_{t}\frac{\langle X_t, \theta\rangle}{\sigma_t}+{\ell}_{t,1} & \text{if } {z}_{t}(\theta) > \tau_{t},\\
\tau_{t}\frac{\langle X_t, \theta\rangle}{\sigma_t}+{\ell}_{t,2} & \text{if } {z}_{t}(\theta) <-\tau_{t},
\end{cases}
\end{align}
where
\begin{align*}
    &{\ell}_{t,1}=\mathds{1}_{\{\mu(S\left\lVert X_t\right\rVert_2) >\mu_{t,1}\}\cap\{\mu(-S\left\lVert X_t\right\rVert_2) <\mu_{t,1}\}} \left(\frac{\tau_t \mu^{-1}(\mu_{t,1})}{\sigma_t}+{f}_t(\mu_{t,1})\right),\\
    &{\ell}_{t,2}=\mathds{1}_{\{\mu(S\left\lVert X_t\right\rVert_2) >\mu_{t,2}\}\cap\{\mu(-S\left\lVert X_t\right\rVert_2) <\mu_{t,2}\}} \left(\frac{\tau_t \mu^{-1}(\mu_{t,2})}{\sigma_t}+{f}_t(\mu_{t,2})\right),\\
    & {f}_t(z)=\frac{-r_t z+\int_{0}^z \mu(x)dx}{\sigma_t^2},\\
    & {\mu}_{t,1}=r_t-\tau_t\sigma_t ,\ {\mu}_{t,2}=r_t+\tau_t\sigma_t,\\
    & \sigma_t  \ \text{is a robustness parameter that will be settled later.}
\end{align*}

To understand the behavior of the extended Huber loss, we examine its gradient:
\begin{align}\label{extended_huber_loss_gradient}
\nabla{\ell}_{t}(\theta) = 
\begin{cases}
-{z}_{t}(\theta)\frac{X_{t}}{\sigma_{t}} & \text{if } |{z}_{t}(\theta)| \leq \tau_{t}, \\
-\tau_{t}\frac{X_{t}}{\sigma_{t}} & \text{if } {z}_{t}(\theta) > \tau_{t}, \\
\tau_{t}\frac{X_{t}}{\sigma_{t}} & \text{if } {z}_{t}(\theta) < -\tau_{t}.
\end{cases}
\end{align}

In this sense, such loss shows great similarity with classical Huber loss: its gradient grows linearly with the bias $z_t(\theta)$, and becomes constant when the absolute value of $z_t(\theta)$ appears too large. As a result, extended Huber loss manages to inherit the robustness properties: it acts "approximately" quadratic within a short range of bias and linear on the outside, similarly as the original Huber loss \citep{huber1964robust}.

\section{Extended Huber Loss Method in GLB Problem}
In this section, we officially introduce our algorithm based on the loss proposed in the previous Section~\ref{section_preliminaries}, namely the Generalized Linear Bandit with Extended Huber loss Method (GLB-EHM), as shown in Algorithm~\ref{alg:EHM}.

\begin{algorithm}[H]
\caption{Generalized Linear Bandit with Extended Huber loss Method (GLB-EHM)}
\label{alg:EHM}
\begin{algorithmic}[1]
\Input time horizon $T$, moment order $\varepsilon$, feature bound $S$, Lipschitz constant $L$, positive parameters $\delta,\lambda,\sigma_{\min},\tau_0$ which will be chosen later
\State set $k=d\log(1+\frac{L^2T}{\sigma_{\min}^2\lambda d}),V_1 = \lambda I_d, {\theta}_1 = 0$ and calculate $\beta_1$
\For{$t = 1, 2, \ldots, T$}
    \State select $X_t = \arg\max_{x \in \mathcal{X}_t} \left\{ \langle x, {\theta}_t \rangle + \beta_{t} \|X\|_{V_{t}^{-1}} \right\}$
    \State receive reward $r_t$ and $\nu_t$
    \State set $\alpha_t=\frac{1}{\mu'(X_t^T\theta_t)^2},\eta_t = 2(1+KS)\mu'(X_t^T\theta_t)$ 
    \State set $\sigma_t = \max \left\{ \nu_t,\; \sigma_{\min},\; \sqrt{\frac{2\mu'(X_t^T\theta_t)L\beta_{t}}{\tau_0  t^\frac{1-\varepsilon}{2(1+\varepsilon)}}} \|X_t\|_{V_{t}^{-1}} \right\}$
    \State set $\tau_t = \tau_0 \frac{\sqrt{1+w_t^2}}{w_t} t^{\frac{1-\varepsilon}{2(1+\varepsilon)}}$, where $w_t = \frac{1}{\sqrt{\alpha_t}} \left\| \frac{X_t}{\sigma_t} \right\|_{V_{t}^{-1}}$
    \State update $V_{t+1} = V_{t} + \alpha_t^{-1} \sigma_t^{-2} X_t X_t^\top$
    \State calculate ${\theta}_{t+1}$ and $\beta_{t+1}$
\EndFor
\end{algorithmic}
\end{algorithm}

It is worth noting that $\beta_t$ is the confidence region constructed in each round $t$. Before presenting the main theorems for the regret analysis, several computational details need to be clarified. The update rules for $\theta_{t+1}$ and $\beta_t$ are given by
\begin{align}
&\theta_{t+1} = \arg\min_{\theta \in \Theta} \left\{ \langle \theta, \nabla{\ell}_{t}(\theta_t)\rangle + \frac{1}{2\eta_t} || \theta-\theta_t  ||_{V_{t+1}}^2 \right\}\label{OMD update}  ,\\
&\beta_t = 54(1+KS)\tau_0(t-1)^{\frac{1-\epsilon}{2(1+\epsilon)}}\left( \sqrt{\log(\frac{2T^2}{\delta})}+\frac{1}{15}\right)^2+\sqrt{\lambda}S\label{confidence region},
\end{align}
where the parameter $\eta_t$ is the learning rate. The computational efficiency of our algorithm, which has been indeed guaranteed in \cite{wang2025heavy}: with the aid of projection decomposition and the Sherman-Morrison formula, online mirror descent step \eqref{OMD update} can be realized within time cost complexity $\mathcal{O}(d^3)$. Furthermore, our algorithm performs a one-pass update, meaning that it does not need to store all historical data.

Finally, we present key theorems that are critical for error estimation and the sublinear regret bound analysis.
\begin{theorem}\label{thm:confidence_region_bound}
If we set
\begin{align*}
&\sigma_t=\max\left\{\nu_t,\sigma_{\min},\sqrt{\frac{2\mu'(X_t^T\theta_t)L\beta_{t}\left\|X_t\right\|_{V_{t}^{-1}}^2}{\tau_0t^{\frac{1-\varepsilon}{2(1+\varepsilon)}}}}\right\},
\end{align*}
then for any $\delta \in (0,1)$, with probability at least $1-3\delta$, for any $t \geq 1$, we have
\begin{align*}
\left\|\theta_{t+1}-\theta_{*}\right\|_{V_{t+1}} \leq  54(1+KS)\tau_0t^{\frac{1-\epsilon}{2(1+\epsilon)}}\left( \sqrt{\log(\frac{2T^2}{\delta})}+\frac{1}{15}\right)^2+\sqrt{\lambda}S=\beta_{t+1}.
\end{align*}
\end{theorem}
Having established a high-probability of bound on confidence region, we now proceed to deal with regret analysis.

\begin{theorem}\label{thm:regret_bound}
By setting 
\begin{align*}
\sigma_t=\max\left\{\nu_t,\sigma_{\min},\sqrt{\frac{2\mu'(X_t^T\theta_t)L\beta_{t}\left\|X_t\right\|_{V_{t}^{-1}}^2}{\tau_0t^{\frac{1-\varepsilon}{2(1+\varepsilon)}}}}\right\},~\sigma_{\min}=\frac{1}{\sqrt{T}},~\delta=\frac{1}{3T}, 
\end{align*}
with probability at least $1-\frac{1}{T}$, we obtain 
\begin{align*}
\mathcal{R}(T)  
\leq \widetilde{\mathcal{O}}\left(d(1+KS) T^{\frac{1-\epsilon}{2(1+\epsilon)}} \sqrt{\left(\sum_{t=1}^{T}\nu_t^2\right)+1} \right).
\end{align*}
In other words, the expectation of regret is bounded by
\begin{align*}
    \mathbb{E}[\mathcal{R}(T)] 
    &\leq  (1-3\delta) \cdot \widetilde{\mathcal{O}}\left(d(1+KS) T^{\frac{1-\epsilon}{2(1+\epsilon)}} \sqrt{\left(\sum_{t=1}^{T}\nu_t^2\right)+1} \right)+3\delta T\\
    & = \widetilde{\mathcal{O}}\left(d(1+KS) T^{\frac{1-\epsilon}{2(1+\epsilon)}} \sqrt{\left(\sum_{t=1}^{T}\nu_t^2\right)+1} \right),    
\end{align*}
where $\widetilde{\mathcal{O}}(\cdot)$ indicates that logarithmic factors with respect to $T$ have been hidden. 
\end{theorem}  

\begin{corollary}\label{cor:regret_bound}
If  $\nu_t$ can be controlled by a constant $\nu$, with probability at least $1-1/T $, we obtain
\begin{align*}
\mathcal{R}(T) \leq  \widetilde{\mathcal{O}}\left( d(1+KS)\nu T^{\frac{1}{1+\epsilon}}\right).
\end{align*}    
\end{corollary}

Compared to the regret lower bound $\Omega(dT^{\frac{1}{1+\epsilon}})$ in LB problem \citep{shao2018almost}, our regret bound is nearly optimal. Moreover, it removes the dependency on the sensitive parameter $\kappa$ in the leading coefficient.

%
\begin{figure}[H]
    \centering
    \includegraphics[width=0.8\textwidth]{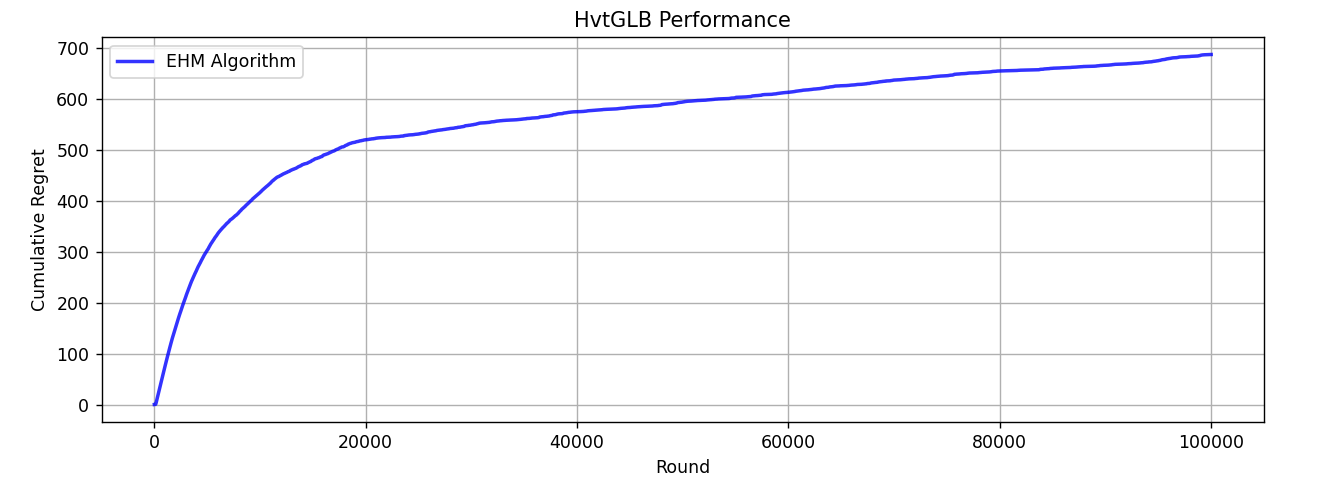}
    \caption{GLB-EHM Performance}
    \label{fig:GLB-EHM_performance}
\end{figure}

We tested GLB-EHM in logit model with noise following student's t-distribution (see Fig.~\ref{fig:GLB-EHM_performance}). The average run time was $11.25$ seconds.

\section{Piecewise GLB} \label{sec:PGLB-EHM}
In this section, we discuss the GLB problem under a different setting, namely, the piecewise GLB. The contextual character $\theta_*$ is no longer constant on the whole action set, instead, it appears as different constants in different areas. This variation is quite meaningful to our work, it not only slightly weakens the structure in original GLB problem, but gives us a new approach to look deeper into nonlinear bandits which will be explained later.

Formally, we consider the following problem, where the reward at time $t$ is given by
\begin{align*}
     r_t=\mu_{(X_t)}(\langle X_t, \theta_*(X_t)\rangle)+\varepsilon_t,
\end{align*}
where $\mu_{(X_t)}=\mu_i$, $\theta_*(X)=\theta_i^*$ (an unknown constant) if $X \in \mathcal{A}^i$, $i=1,2,..,m$. 

The collection $\{\mathcal{A}^i\}_{i=1}^{m}$ satisfies (i) $\forall i \neq j$, $\mathcal{A}^i\cap \mathcal{A}^j=\varnothing$; (ii) $\mathcal{X}_t \subset\cup_{i=1}^{m} \mathcal{A}^i \subset B_d$ for any $t=1,2,..T$. Other conditions and notations remain unchanged as original GLB in Subsection~\ref{GLB with hvt noise}. Feature bound in Assumption~\ref{assumption_bd} stands for each $\theta_i^*$ while Assumption~\ref{assumption_linkfunc} stands for every $\mu_i$. Beyond these, an extra assumption is required in our theoretical analysis. Let us define $\mu_t^i \triangleq \max_{X \in \mathcal{X}_t \cap \mathcal{A}^i} \mu_i(X^T\theta_i^*)$ and $\mu_t^* = \max_{1 \leq i \leq m} \mu_t^i$.

\begin{assumption} \label{assumption_pa_gap}
    For any $t=1,2,..,T$ and $i=1,2,..,m$, we have $\mu_t^*-  \mu_t^i \geq a$ if $\mu_t^i \neq \mu_t^*$, where $a >0$ is a known constant.
\end{assumption}

This assumption means that the optimal reward in the best piece is better than the reward obtained in any suboptimal piece by at least a known constant $a>0$ (note that this constant always exists since we have finite pieces only). The central idea of PGLB-EHM, shown in Algorithm~\ref{alg:PGLB-EHM}, is to maintain separate estimations in each area $\mathcal{A}^i$ in each round $t$. If $\mathcal{A}^i \cap \mathcal{X}_t = \varnothing$, we  set $X_t^i = \varnothing$ to skip the estimation in $\mathcal{A}^i$. Remember that regret in this case can be written as
\begin{align} \label{PAGLB_regret}
    \mathcal{R}(T)=\sum_{t=1}^{T} [\mu(\langle X_t^*, \theta_*(X_t^*)\rangle)-\mu(\langle X_t, \theta_*(X_t)\rangle)],
\end{align}
where $X_t^*= \mathop{\arg\max}_{X \in \mathcal{X}_t} \mu(\langle X, \theta_*(X)\rangle)$ stands for the optimal arm at time $t$. Subsequently we give a regret upper bound of PGLB-EHM, with the runtime performance is shown in Fig.~\ref{fig:PGLB-EHM_performance}.

\begin{theorem}\label{thm:PGLB-EHM regret bound}
   By setting 
\begin{align*}
\sigma_{\min}=\frac{1}{\sqrt{T}},\delta=\frac{1}{3mT} 
\end{align*}
and if $v_t \leq \nu$ for any $t$, with probability at least, $1-\frac{1}{T}$ we obtain 
\begin{align*}
\mathcal{R}(T)  
\leq \widetilde{\mathcal{O}}\left(m(1+\frac{2L^2S}{a\kappa}) d(1+KS)\nu T^{\frac{1}{1+\epsilon}} \right).
\end{align*}
\end{theorem}

\begin{algorithm}[h]
\caption{Piecewise Generalized Linear Bandits with Extended Huber loss Method (PGLB-EHM)}
\label{alg:PGLB-EHM}
\begin{algorithmic}[1]
\Input time horizon $T$, parameters $\delta,\lambda,\sigma_{\min},S,L,\tau_0, \varepsilon$
\State set $k=d\log(1+\frac{L^2T}{\sigma_{\min}^2\lambda d})$ and calculate $\beta_1$
\For{$i=1,2,\ldots,m$}
    \State set $t_i=1, V_1^i = \lambda I_d, {\theta}_1^i = 0$ 
\EndFor
\For{$t = 1, 2, \ldots, T$}
    \For{$i = 1, 2, \ldots, m$}
    \State calculate $X_t^i = \arg\max_{x \in \mathcal{A}^i \cap\mathcal{X}_t} \left\{ \langle x, {\theta}_{t_i}^i \rangle + \beta_{t_i} \|x\|_{(V_{t_i}^{i})^{-1}} \right\}$
    \EndFor
    \State set $j= \arg\max_{1\leq i \leq m} \left\{ \mu_i\big( \langle X_t^i, {\theta}_{t_i}^i \rangle + \beta_{t_i} \|X_t^i\|_{(V_{t_i}^{i})^{-1}} \big)\right\}$
    \State select $X_t = X_t^j$
    \State receive reward $r_t$ and $\nu_t$
    \State set $V_t=V_{t_j}^j, \theta_t=\theta_{t_j}^j$
    \State set $\alpha_t=\frac{1}{\mu'(X_t^T\theta_t)^2},\eta_t = 2(1+KS)\mu'(X_t^T\theta_t)$ 
    \State set $\sigma_t = \max \left\{ \nu_t,\; \sigma_{\min},\; \sqrt{\frac{2\mu'(X_t^T\theta_t)L\beta_{t_j}}{\tau_0  t_j^\frac{1-\varepsilon}{2(1+\varepsilon)}}} \|X_t\|_{V_{t}^{-1}} \right\}$
    \State set $\tau_t = \tau_0 \frac{\sqrt{1+w_t^2}}{w_t} t_j^{\frac{1-\varepsilon}{2(1+\varepsilon)}}$, where $w_t = \frac{1}{\sqrt{\alpha_t}} \left\| \frac{X_t}{\sigma_t} \right\|_{V_{t}^{-1}}$
    \State update $V_{t_j+1}^j = V_{t} + \alpha_t^{-1} \sigma_t^{-2} X_t X_t^\top$
    \State calculate ${\theta}_{t_j+1}^j$ and $\beta_{t_j+1}$ by \eqref{OMD update} and \eqref{confidence region} 
    \State $t_j \leftarrow t_j+1$
\EndFor
\end{algorithmic}
\end{algorithm}

\begin{figure}[H]
    \centering
    \includegraphics[width=0.8\textwidth]{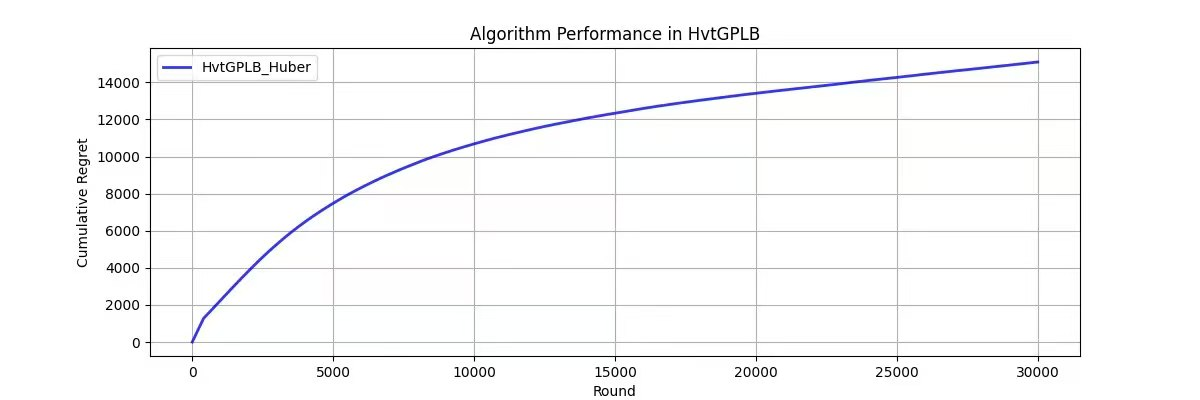}
    \caption{PGLB-EHM Performance}
    \label{fig:PGLB-EHM_performance}
\end{figure}

\section{Nonlinear Bandit} \label{SNB-EHM}

\subsection{A Special Form of NB}
Inspired from PGLB in Section~\ref{sec:PGLB-EHM}, we continue to consider the problem when $\theta_*(\cdot): \mathbb{R}^d \rightarrow \mathbb{R}^d$ is no longer piecewise constant, but becomes an unknown function. As a result, we are discussing a special form of nonlinear bandit, which fits the form
\begin{align}
    r_t = b(X_t) + \varepsilon_t,
\end{align}
where $b:\mathbb{R}^d \rightarrow \mathbb{R}$ is unknown and the feasible set becomes the closed unit ball in $\mathbb{R}^d$. The optimal point 
\begin{align*}
 X_t^*=\mathop{\arg\max}_{X \in B_d} b(X)=X^*
\end{align*}
is identical in every round~$t$. Our goal is to minimize regret
\begin{align*}
    \mathcal{R}(T)=\sum_{t=1}^{T} b(X^*)-b(X_t).
\end{align*}
In this section we assume $b(x)=\mu(x^T\theta_*(x))$, while in the next subsection, we  explain the feasibility of transforming it into common nonlinear bandit by a lifting approach.

Again we should give some restrictions on $b$ and $\theta_*(\cdot)$ to provide convenience for our analysis beyond Assumption~\ref{assumption_bd} and ~\ref{assumption_linkfunc}.
\begin{assumption} \label{assumption_SNB Lip continuity}
    $\theta_*(\cdot)$ is $L'$-Lipschitz continuous.
\end{assumption}

\begin{assumption} \label{assumption_SNB gap}
  For any $X \in B_d$, we have $ b(X^*)-b(X) \geq C_1 ||X^*-X||_\infty^{C_2}$ where $C_1,C_2 > 0$.
\end{assumption}

Assumption~\ref{assumption_SNB Lip continuity} guarantees expected reward does not alter intensely inside an arbitrary small area. 
Assumption~\ref{assumption_SNB gap} tells us the difference between optimal point and any other point can be estimated by an index (namely $\infty$-norm). It also indicates the uniqueness of $X^*$, which reduces analytical difficulty in exploring cost and is reasonable in practical scenarios.

Meanwhile, we will adpot bisection method to refine feasible space $B_d$ and explore on those refined subregions. More precisely, we will solve GLB problem on small areas and give error estimations between assumed GLB structure and real NB. Before that, we need to present some concepts that are useful to comprehend detail skills in the question and proofs.

\begin{definition}\label{def_father and son region}
For a region $A=\prod_{i=1}^{d}[a_i,b_i] \subset \mathbb{R}^d$, using bisection method on each coordinate will produce $2^d$ subregions namely $A_n(n=1,2,..,2^d)$ in total. We call the single process of generating $2^d$ subregions through bisection is an \textbf{operation}. $A$ is \textbf{father} of $A_n$, and $A_n$ is a \textbf{son} of $A$. Moreover, if $A$ and $B$ can be related through father-son chain and $A \subset B$, we call $A$ is \textbf{descendant} of $B$, while $B$ is \textbf{predecessor} of $A$.
\end{definition}

\begin{definition}\label{def_width}
 For a region $A=\prod_{i=1}^{d}[a_i,b_i] \subset \mathbb{R}^d$, its width $d(A) \triangleq \max_{1 \leq i \leq d} {|b_i-a_i|}$.   
\end{definition}

\begin{definition}\label{def_width}
 For a region $A$, $t(A)-1$ is the total exploration time on $A$. It is worth noting that we do not count exploration on either $A$'s predecessor or $A$'s descendant into $t(A)$. The exploration time limit on $A$ is $T(A)\triangleq 1/d(A)^2$. 
\end{definition}

This setting ensures that the total exploration budget increases as the region becomes finer. Recall that the feasible set is $B_d= [-1,1]^d \cap B_d$. We denote by $C_d = [-1,1]^d$. Finally, we initialize $t(A)=1$,$\theta_1^A=0$ and $V_1^A = \lambda I_d$ if $A$ is first added to $\mathcal{A}_0$.

\begin{algorithm}[!h]
\caption{Nonlinear Bandit with Extended Huber loss Method (NB-EHM)}
\label{alg:SNB-EHM}
\small
\begin{algorithmic}[1]
\Input time horizon $T$, parameters $\delta,\lambda,\sigma_{\min},S,L,\varepsilon$
\State set $k=d\log(1+\frac{L^2T}{\sigma_{\min}^2\lambda d})$ and calculate $\beta_1=(\sqrt{\lambda}+4kL)S+dL'^2$
\State set $\mathcal{A}_0 = \{A|\text{$A$ is son of $C_d$} \}$.
\For{$t = 1, 2, \ldots, T$}
\For{$A \in \mathcal{A}_0$}
    \State calculate $X_t^A = \arg\max_{x \in A \cap B_d} \left\{ \langle x, {\theta}_{t(A)}^A \rangle + \beta_{t(A)}^A \|x\|_{(V_{t(A)}^{A})^{-1}} \right\}$
    \EndFor
    \State set $\mathbf{A} = \arg\max_{A \in \mathcal{A}_0 } \left\{ \langle X_t^A, {\theta}_{t(A)}^A \rangle + \beta_{t(A)}^A \|X_t^A\|_{(V_{t(A)}^{A})^{-1}} \right\}$
    \State select $X_t = X_t^{\mathbf{A}}$
    \State receive reward $r_t$ and $\nu_t$
    \State set $V_t=V_{t(\mathbf{A})}^\mathbf{A}, \theta_t=\theta_{t(\mathbf{A})}^\mathbf{A}$
    \State set $\alpha_t=\frac{1}{\mu'(X_t^T\theta_t)^2},\eta_t = 2(1+KS)\mu'(X_t^T\theta_t)$
    \State set $\tau_0(\mathbf{A})=\max\left\{2\sqrt{2d}L'd(\mathbf{A}),\frac{\sqrt{2k}(\log 3T)^{\frac{1-\varepsilon}{2(1+\varepsilon)}}}{(\log \frac{2T^2}{\delta})^{\frac{1}{1+\varepsilon}}}\right\} $
    \State set $\sigma_t = \max \left\{ 2\nu_t,\; \sigma_{\min},\; \sqrt{\frac{2\mu'(X_t^T\theta_t)L\beta_{t(\mathbf{A})}^\mathbf{A}}{(1+KS)\tau_0(\mathbf{A})  t(\mathbf{A})^\frac{1-\varepsilon}{2(1+\varepsilon)}}} \|X_t\|_{V_{t}^{-1}} \right\}$
    \State set $\tau_t = \tau_0(\mathbf{A}) \frac{\sqrt{1+w_t^2}}{w_t} t(\mathbf{A}) ^{\frac{1-\varepsilon}{2(1+\varepsilon)}}$, where $w_t = \frac{1}{\sqrt{\alpha_t}} \left\| \frac{X_t}{\sigma_t} \right\|_{V_{t}^{-1}}$
    \State update $V_{t(\mathbf{A}) +1}^\mathbf{A} = V_{t} + \alpha_t^{-1} \sigma_t^{-2} X_t X_t^\top$
    \State calculate ${\theta}_{t(\mathbf{A}) +1}^\mathbf{A} $ and $\beta_{t(\mathbf{A}) +1}^\mathbf{A}$ by \eqref{update_NB-EHM_1} and \eqref{update_NB-EHM_2}
    \State $t(\mathbf{A})  \leftarrow t(\mathbf{A}) +1$
    \If{$t(\mathbf{A}) = T(\mathbf{A})$}
        \State perform an \textbf{operation} on $\mathbf{A}$
        \State $\mathcal{A}_0 \leftarrow \mathcal{A}_0 \setminus \mathbf{A}$
        \For{$B \in \left\{  A|\text{$A$ is son of $\mathbf{A}$}\right\}$}
            \State $\mathcal{A}_0 \leftarrow \mathcal{A}_0 \cup \{B\}$
            \State set $t(B)=1$, $\theta_1^B=0$ and $V_1^B=\lambda I_d$
            \EndFor
        \EndIf
\EndFor
\end{algorithmic}
\end{algorithm}
The update rules for $\theta_{t(\mathbf{A})+1}^{\mathbf{A}}$ and $\beta_{t(\mathbf{A})+1}^{\mathbf{A}}$ are given by
\begin{align}
&\theta_{t(\mathbf{A})+1}^{\mathbf{A}} = \arg\min_{\theta \in \Theta} \left\{ \langle \theta, \nabla{\ell}_{t}(\theta_t)\rangle + \frac{1}{2\eta_t} || \theta-\theta_t  ||_{V_{t+1}}^2 \right\}  , \label{update_NB-EHM_1}\\
&\beta_{t(\mathbf{A})+1}^{\mathbf{A}} = 27(1+KS)\tau_0(\mathbf{A})(t(\mathbf{A})-1)^{\frac{1-\epsilon}{2(1+\epsilon)}}\left( \sqrt{\log(\frac{2T^2}{\delta})}+\frac{1}{12}\right)^2+(\sqrt{\lambda}+4kL)S+dL'^2. \label{update_NB-EHM_2}
\end{align}

Now we give the regret upper bound of NB-EHM.

\begin{theorem}\label{thm:SNB-EHM regret bound}
   By setting 
\begin{align*}
\sigma_{\min}=1,\delta=\frac{1}{6T} 
\end{align*}
and if $v_t \leq \nu$ for any $t$,  we attain 
\begin{align*}
    \mathbb{E}[\mathcal{R}(T)] &\leq  \widetilde{O}\left(T^{\alpha}\right),
\end{align*}
where  $\alpha = \frac{\gamma+1}{\gamma+2}+\frac{(d+2)(1-\varepsilon)}{2(\gamma+2)(1+\varepsilon)} $, $\gamma=(d+2)(C_2-1)+1 $ and $\varepsilon >\frac{d}{d+4}$.
\end{theorem}

Here are some numerical results on the sublinear regret and convergence to the optimal point are as shown in Fig.~\ref{fig:SNB-EHM_performance}.
\begin{figure}[H]
  \centering
  \begin{subfigure}{\textwidth}
    \centering
    \includegraphics[width=0.7\linewidth]{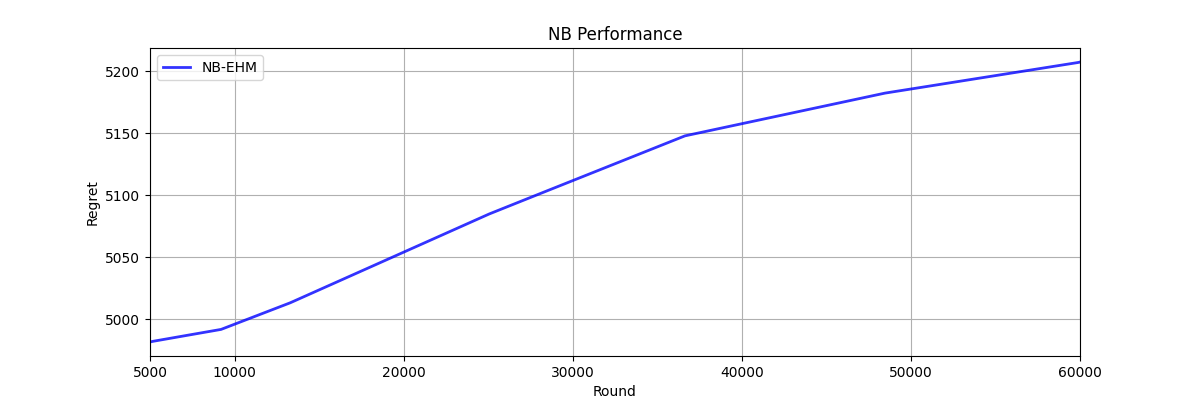}
    \caption{Regret}
  \end{subfigure}
  \hfill 
  \begin{subfigure}{\textwidth} 
    \centering
    \includegraphics[width=0.7\linewidth]{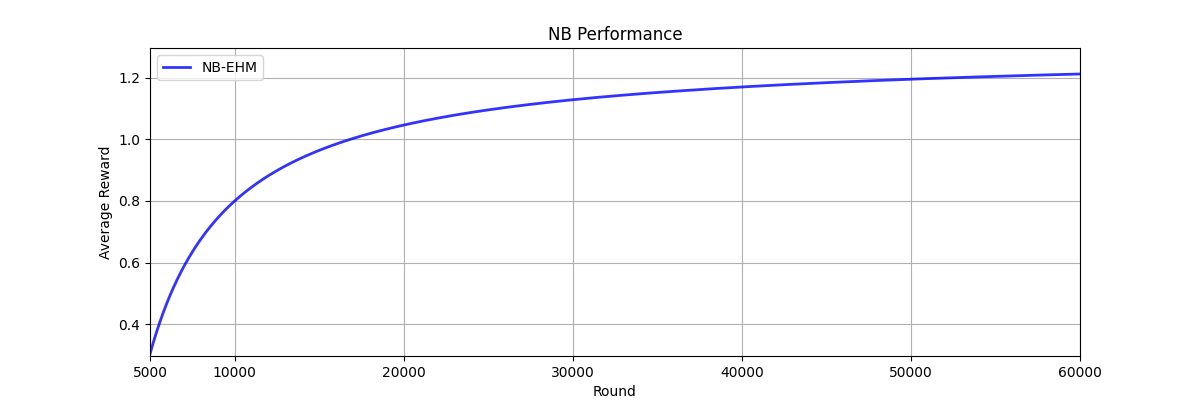}
    \caption{Average Reward}
  \end{subfigure}
  \caption{NB-EHM Performance 1}
  \label{fig:SNB-EHM_performance}
\end{figure}

\subsection{Affine Lifting and Generalized Form }\label{NB-EHM}
For common function $b:\mathbb{R}^d \rightarrow \mathbb{R}$, we can rewrite it in the form 
\begin{align*}
    b(x)=\mu(x^T\theta_*(x)+b_1(x))=\mu(\widetilde{x}^T\widetilde{\theta}_*(\widetilde{x})),
\end{align*}
where 
\begin{align*}
\widetilde{x}= \frac{1}{\sqrt{2}}
\begin{pmatrix}
x \\
1
\end{pmatrix},
\quad 
\widetilde{\theta}_*(\widetilde{x})= \sqrt{2}
\begin{pmatrix}
\theta_1^{T}(x) , b_1(x)
\end{pmatrix}^T.
\end{align*}
It is easy to check $\widetilde{x} \in B_{d+1}$. We will need one revised version of Assumption~\ref{assumption_bd} and~\ref{assumption_SNB Lip continuity} for this general nonlinear bandit problem. 

\begin{assumption}
    In each round $t$, $\mathcal{X}_t=B_d$, $||\widetilde{\theta}_*(\cdot)||_2 \leq S'$ and $\widetilde{\theta}_*(\cdot)$ is $L'$-Lipschitz continuous.
\end{assumption}

The rest assumptions are Assumptions~\ref{assumption_linkfunc} and ~\ref{assumption_SNB gap}.We need to replace $S,d$ with $S',d+1$ to make NB-EHM (Algorithm~\ref{alg:SNB-EHM}) work in generalized NB problem. Consequently, through similar proof as mentioned in Appendix~\ref{pf_SNB}, we have the following regret bound.

\begin{corollary}\label{thm:NB-EHM regret bound}
   By setting 
\begin{align*}
\sigma_{\min}=1,\delta=\frac{1}{6T} 
\end{align*}
and if $v_t \leq \nu$ for any $t$,  we attain 
\begin{align*}
    \mathbb{E}[\mathcal{R}(T)] &\leq  \widetilde{O}\left(T^{\alpha}\right),
\end{align*}
where  $\alpha = \frac{\gamma+1}{\gamma+2}+\frac{(d+3)(1-\varepsilon)}{2(\gamma+2)(1+\varepsilon)} $, $\gamma=(d+3)(C_2-1)+1 $ and $\varepsilon >\frac{d+1}{d+5}$.
\end{corollary}
Here are some numerical results on the sublinear regret and convergence to the optimal point as shown in Fig.~\ref{fig:NB-EHM_performance}.

\begin{figure}[h]
  \centering
  \begin{subfigure}{\textwidth}
    \centering
    \includegraphics[width=0.7\linewidth]{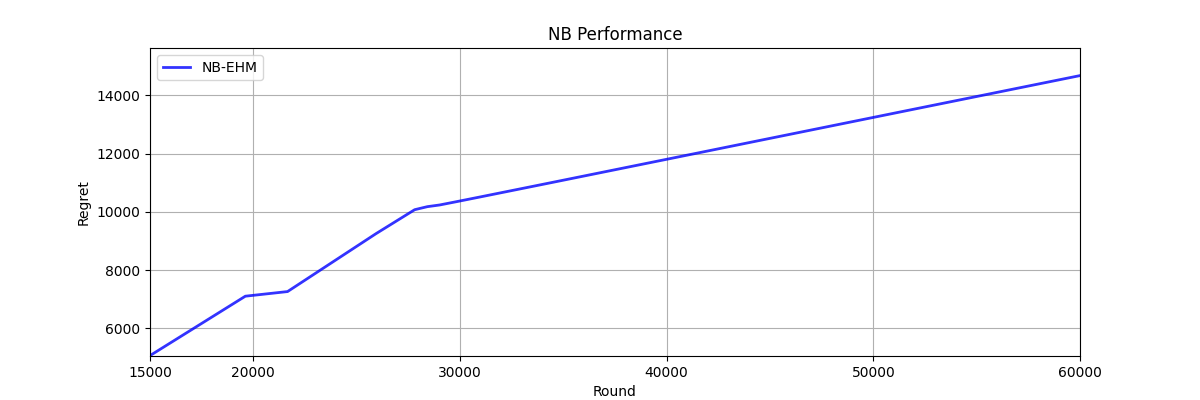}
    \caption{Regret}
  \end{subfigure}
  \hfill 
  \begin{subfigure}{\textwidth} 
    \centering
    \includegraphics[width=0.7\linewidth]{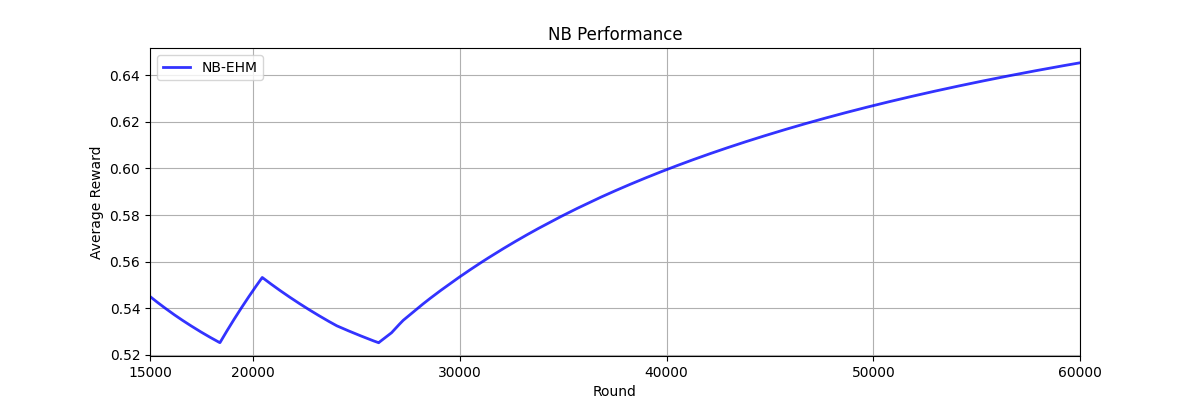}
    \caption{Average Reward}
  \end{subfigure}
  \caption{NB-EHM Performance 2}
  \label{fig:NB-EHM_performance}
\end{figure}

\section{Conclusion and Future Work}
In the heavy-tailed GLB, we proposed the EHM algorithm based on the adaptive Huber loss, which has already performed its robustness in the LB problem. In the policy, we used a special parameter to eliminate the presence of another susceptible parameter both in the algorithm and the leading coefficient. The regret upper bound is consistent with lower bound $\Omega(dT^{1/(1+\varepsilon)})$ in \cite{shao2018almost} except for logarithmic terms, thus is nearly optimal. Next, we study GLB with contextual characteristic that is piecewise constant and present PGLB-EHM. Inspired from PGLB-EHM's analysis, we go on to discuss a special case of NB. With the help of bisection method and exploration restriction, we manage to introduce NB-EHM that achieves a sublinear regret. Also we proved that the general NB can be transformed into the former special case by a lifting approach. We showed numerical experiments for these algorithms to illustrate their run-time performance in addition to the theoretical results. In the future, we will seek effective ways to establish regret lower bound and design order-optimal algorithms for the general NB problem under various assumptions. Moreover, Assumption~\ref{assumption_SNB Lip continuity} may be relaxed to Holder continuity and Assumption~\ref{assumption_SNB gap} may be relaxed to the situation when optimal points make up a finite or even infinite set by using stronger theoretical tools.


\bibliography{sample}
\newpage
\appendix

\section{Denoising Huber Loss}\label{denoising_huber_loss}

Before controlling estimation error, we first introduce the denoising version of Huber Loss, which paves the way for further regret analysis. It is worth noting that, by substituting $\mu(\langle X_t, \theta_*\rangle)$ for $r_t$, we can replace $\ell_t(\theta),z_t(\theta),f_t(z),\ell_{t,1},\ell_{t,2},\mu_{t,1},\mu_{t,2}$ with $\widetilde{\ell_t}(\theta),\widetilde{z_t}(\theta),\widetilde{f_t}(z),\widetilde{\ell}_{t,1},\widetilde{\ell}_{t,2},\widetilde{\mu}_{t,1},\widetilde{\mu}_{t,2}$, which explicitly gives:
\[
\widetilde{\ell}_{t}(\theta) :=
\begin{cases}
\frac{-\mu(\langle X_t, \theta_*\rangle)\langle X_t, \theta\rangle+\int_{0}^{\langle X_t, \theta\rangle} \mu(x)dx}{\sigma_t^2} & \text{if } |\widetilde{z}_{t}(\theta)| \leq \tau_{t}, \\
-\tau_{t}\frac{\langle X_t, \theta\rangle}{\sigma_t}+\widetilde{\ell}_{t,1} & \text{if } \widetilde{z}_{t}(\theta) > \tau_{t},\\
\tau_{t}\frac{\langle X_t, \theta\rangle}{\sigma_t}+\widetilde{\ell}_{t,2} & \text{if } \widetilde{z}_{t}(\theta) <-\tau_{t},
\end{cases}
\]
where 
\begin{align*}
    &\widetilde{z}_{t}(\theta)=\frac{\mu(\langle X_t, \theta_*\rangle)-\mu(\langle X_t, \theta\rangle)}{\sigma_t},\\
    &\widetilde{\ell}_{t,1}=\mathds{1}_{\{\mu(S\left\lVert X_t\right\rVert_2) >\mu_{t,1}\}\cap\{\mu(-S\left\lVert X_t\right\rVert_2) <\mu_{t,1}\}} \left(\frac{\tau_t \mu^{-1}(\mu_{t,1})}{\sigma_t}+\widetilde{f}_t(\mu_{t,1})\right),\\
    &\widetilde{\ell}_{t,2}=\mathds{1}_{\{\mu(S\left\lVert X_t\right\rVert_2) >\mu_{t,2}\}\cap\{\mu(-S\left\lVert X_t\right\rVert_2) <\mu_{t,2}\}} \left(\frac{\tau_t \mu^{-1}(\mu_{t,2})}{\sigma_t}+\widetilde{f}_t(\mu_{t,2})\right),\\
    & \widetilde{f}_t(z)=\frac{-\mu(\langle X_t, \theta_*\rangle)z+\int_{0}^z \mu(x)dx}{\sigma_t^2},\\
    & \widetilde{\mu}_{t,1}=\mu(\langle X_t, \theta_*\rangle)-\tau_t\sigma_t ,\widetilde{\mu}_{t,2}=\mu(\langle X_t, \theta_*\rangle)+\tau_t\sigma_t.
\end{align*}

It is not difficult to verify that both $\widetilde{\ell}_t(\theta)$ and $\ell_t(\theta)$ are well defined. Through simple computation, we can prove that the gradient and Hessian of the denoising loss function is given by:
\[
\nabla\widetilde{\ell}_{t}(\theta) = 
\begin{cases}
-\widetilde{z}_{t}(\theta)\frac{X_{t}}{\sigma_{t}} & \text{if } |\widetilde{z}_{t}(\theta)| \leq \tau_{t}, \\
-\tau_{t}\frac{X_{t}}{\sigma_{t}} & \text{if } \widetilde{z}_{t}(\theta) > \tau_{t}, \\
\tau_{t}\frac{X_{t}}{\sigma_{t}} & \text{if } \widetilde{z}_{t}(\theta) < -\tau_{t},
\end{cases}
\quad
\nabla^{2}\widetilde{\ell}_{t}(\theta) = \mathds{1}\left\{|\widetilde{z}_{t}(\theta)| \leq \tau_{t}\right\}\mu'(\langle X_t, \theta\rangle)\frac{X_{t}X_{t}^{\top}}{\sigma_{t}^{2}}.
\]
\begin{property}[optimal point]  \label{property:huber_loss}
$\theta_* = \mathop{\arg\min}_{\theta \in \Theta}\widetilde{\ell}_{t}(\theta).$
\end{property}
\begin{proof}
Utilizing Taylor expansion at $\theta_*$,  for $\forall \theta \in \Theta$, we have 
\begin{align*}
\widetilde{\ell}_{t}(\theta)
&=\widetilde{\ell}_{t}(\theta_*)+\langle \nabla\widetilde{\ell}_{t}(\theta_*), \theta-\theta_* \rangle+\frac{1}{2}\left\lVert \theta-\theta_*\right\rVert_{\nabla^2\widetilde{\ell}_{t}(\xi)}^2\\
&=\widetilde{\ell}_{t}(\theta_*)+\frac{1}{2}\left\lVert \theta-\theta_*\right\rVert_{\mathds{1}\left\{|\widetilde{z}_{t}(\theta)| \leq \tau_{t}\right\}\mu'(\langle X_t, \xi\rangle)\frac{X_{t}X_{t}^{\top}}
{\sigma_{t}^{2}}}^2\\
& \geq \widetilde{\ell}_{t}(\theta_*),
\end{align*}
where $\xi$ lies between $\theta$ and $\theta_*$. The last inequality holds since $\mu'(\cdot)>0$ (Assumption~\ref{assumption_linkfunc}) and $X_tX_t^\top \succeq 0$.
This completes the proof.
\end{proof}

\section{Useful properties of Assumption \ref{assumption_linkfunc}}
\label{useful_prop}

In this section, we simply give some basic conclusions from the previous assumption on link function $\mu$, which appears as strong tools in further proof.
\begin{lemma}[control of derivative] \label{prop_linkfunc}
If $\mu''(\cdot) \leq K\mu'(\cdot)$ in $[-S,S]$, then for $\forall z_1,z_2 \in [-S,S]$:
\begin{align} 
\int_0^1 (1-v)\mu'(z_1+v(z_2-z_1))dv \geq \frac{\mu'(z_1)}{2+K|z_2-z_1|}, \label{prop_linkfunc_1}
\end{align}
\begin{align}
\mu'(z_2)e^{-K|z_2-z_1|} \leq \mu'(z_1) \leq \mu'(z_2)e^{K|z_2-z_1|}. \label{prop_linkfunc_2}
\end{align}
\end{lemma}
The detailed proof of Lemma~\ref{prop_linkfunc} is given in Lemma C.2 of \cite{sawarni2025generalized}. Setting $z_2 = 0$ in \eqref{prop_linkfunc_2}, we obtain the following corollary.

\begin{corollary}\label{corollary_linkfun}
For $\forall z \in [-S,S]$, $\mu'(0)e^{-KS} \leq \mu'(z) \leq \mu'(0)e^{KS}$.
\end{corollary}

\section{Error Estimation Analysis} \label{error_analysis}

Next, we state the main lemmas with respect to our error estimation, controlling the upper bound of EHM. We denote event $A_{t} = \left\{\forall s=1,2,..,t,\ \left\|{\theta}_{s}-\theta_{*}\right\|_{V_{s}} \leq \beta_{s}\right\}$ for simplicity.

\begin{lemma}[general upperbound] \label{lemma_general_upperbound}
If  $\sigma_t \geq \sqrt{2\frac{1}{\sqrt{\alpha_t}}\frac{L\beta_{t}\left\|X_t\right\|_{V_{t}^{-1}}^2}{\tau_0t^{\frac{1-\varepsilon}{2(1+\varepsilon)}}}}$ and $A_t$ holds, then 
\begin{align*}
\left\|{\theta}_{t+1}-\theta_{*}\right\|_{V_{t+1}}^{2}
& \leq   \lambda S^2+ \sum_{s=1}^{t}\eta_s^2\left\|\nabla\ell_{s}(\theta_{s})\right\|_{V_{s+1}^{-1}}^2 +\sum_{s=1}^{t}2\eta_s\langle \nabla\widetilde{\ell}_s(\theta_s)-\nabla{\ell}_s(\theta_s),\theta_s-\theta_* \rangle+\\
&\sum_{s=1}^{t}C_s\lVert \theta_s-\theta_* \rVert_{\frac{X_sX_s^{T}}{\sigma_s^2}}^2,
\end{align*}
where  $C_t=\frac{1}{\alpha_t}-\frac{\eta_t \cdot\mu'(X_t^T \theta_t)}{1+KS}$.
\end{lemma}

\begin{proof}\label{pf_general_upperbound}
If $A_t$ holds, by the Cauchy-Schwartz inequality, it is easy to see that 
\begin{align} 
|\widetilde{z}_t(\theta_t)| 
&\leq \left|\frac{\langle X_t,\theta_t-\theta_* \rangle}{\sigma_t}  \right|\cdot \mu'(\langle X_t,\xi_t\rangle)
\leq\left\|\frac{X_t}{\sigma_t}  \right\|_{V_{t}^{-1}}
\cdot \left\|\theta_t-\theta_*  \right\|_{V_{t}}\cdot L\\
&\leq \frac{\left\|\frac{X_t}{\sigma_t}  \right\|_{V_{t}^{-1}}^2}{w_t}\beta_tL \leq  \frac{\tau_0 t^{\frac{1-\epsilon}{2(1+\epsilon)}}}{2Lw_t\beta_{t}}\sqrt{1+w_t^2}\beta_{t}L =\frac{\tau_t}{2} \leq \tau_t, \label{theta_t lies in quadratic area}
\end{align}
where $\xi_t$ lies between $\theta_t$ and $\theta_*$. Meanwhile, $|\widetilde{z}_t(\theta_*)| =0 \leq \tau_t$. This tell us that line segment with endpoints $\theta_t$ and $\theta_*$ lies in the region where $|\widetilde{z}_t(\cdot)| \leq \tau_t$. 

Recalling the denoising Huber Loss in Appendix~\ref{denoising_huber_loss} and using Taylor expansion at $\theta_t$, we have
\begin{align*} 
\widetilde{\ell}_t(\theta_*)
&=\widetilde{\ell}_t(\theta_t)+\langle \nabla\widetilde{\ell}_t(\theta_t),\theta_*-\theta_t \rangle + \int_0^1 (1-v)\mu'\left(\langle X_t,\theta_t+v(\theta_*-\theta_t)\rangle\right)dv\cdot\lVert \theta_t-\theta_* \rVert_{\frac{X_tX_t^{T}}{\sigma_t^2}}^2\\
&\geq \widetilde{\ell}_t(\theta_t)+\langle \nabla\widetilde{\ell}_t(\theta_t),\theta_*-\theta_t \rangle +\frac{1}{2}
\frac{\mu'( X_t^T\theta_t)}{(1+KS)}\lVert \theta_t-\theta_* \rVert_{\frac{X_tX_t^{T}}{\sigma_t^2}}^2.  \tag{because of \eqref{prop_linkfunc_1}}
\end{align*}
This gives
\begin{align*}
\frac{\mu'(x_t^T \theta_t)}{2(1+KS)}\lVert \theta_t-\theta_* \rVert_{\frac{X_tX_t^{T}}{\sigma_t^2}}^2 \leq    \langle \nabla\widetilde{\ell}_t(\theta_t),\theta_t-\theta_* \rangle.
\end{align*}

Multiplying both sides by $2 \eta_t$, we have
\begin{align} 
& \frac{\eta_t \cdot\mu'(x_t^T \theta_t)}{1+KS} \lVert \theta_t-\theta_* \rVert_{\frac{X_tX_t^{T}}{\sigma_t^2}}^2 
\leq  2\eta_t\langle \nabla\widetilde{\ell}_t(\theta_t),\theta_t-\theta_* \rangle\\
 & =\underbrace{2\eta_t\langle \nabla\widetilde{\ell}_t(\theta_t)-\nabla{\ell}_t(\theta_t),\theta_t-\theta_* \rangle}_{\text{Term~A}} +
\underbrace{2\eta_t\langle \nabla{\ell}_t(\theta_t),\theta_{t+1}-\theta_* \rangle}_{\text{Term~B}}+\underbrace{2\eta_t\langle \nabla{\ell}_t(\theta_t),\theta_t-\theta_{t+1} \rangle}_{\text{Term~C}}.\label{denoising_estimation}
\end{align}

For Term B, by Bregman Approximal Inequality (Lemma 3.2 of \cite{chen1993convergence}), we have 
\begin{align} 
\text{Term~B} & \leq \left\|\theta_{t}-\theta_{*}\right\|_{V_{t+1}}^{2} - \left\|\theta_{t+1}-\theta_{*}\right\|_{V_{t+1}}^{2} - \left\|\theta_{t+1}-\theta_{t}\right\|_{V_{t+1}}^{2}\\
& = \left\|\theta_{t}-\theta_{*}\right\|_{V_{t}}^{2} + \frac{1}{\alpha_t}\lVert \theta_t-\theta_* \rVert_{\frac{X_tX_t^{T}}{\sigma_t^2}}^2
- \left\|\theta_{t+1}-\theta_{*}\right\|_{V_{t+1}}^{2} - \left\|\theta_{t+1}-\theta_{t}\right\|_{V_{t+1}}^{2}.\label{general bound for Tem B}
\end{align}

For Term C, by the Cauchy-Schwartz inequality and AM-GM inequality, we have
\begin{align} \label{general bound for Tem C}
\text{Term~C} \leq 2 \eta_t \left\|\nabla\ell_{t}(\theta_{t})\right\|_{V_{t+1}^{-1}}\left\|\theta_{t+1}-\theta_{t}\right\|_{V_{t+1}} \leq \eta_t^2\left\|\nabla\ell_{t}(\theta_{t})\right\|_{V_{t+1}^{-1}}^2 + \left\|\theta_{t+1}-\theta_{t}\right\|_{V_{t+1}}^2.
\end{align}

Combining \eqref{denoising_estimation}, \eqref{general bound for Tem B} and \eqref{general bound for Tem C}, and summing over $s=1$ to $t$, we obtain
\begin{align*}
\left\|\theta_{t+1}-\theta_{*}\right\|_{V_{t+1}}^{2} 
\leq&
\left\|\theta_{1}-\theta_{*}\right\|_{V_{1}}^{2}+ \sum_{s=1}^{t}\eta_s^2\left\|\nabla\ell_{s}(\theta_{s})\right\|_{V_{s+1}^{-1}}^2 + \sum_{s=1}^{t}2\eta_s\langle \nabla\widetilde{\ell}_s(\theta_s)-\nabla{\ell}_s(\theta_s),\theta_s-\theta_* \rangle+\\
&\sum_{s=1}^{t}(\frac{1}{\alpha_s}-\frac{\eta_s \cdot\mu'(X_s^T \theta_s)}{1+KS})\lVert \theta_s-\theta_* \rVert_{\frac{X_sX_s^{T}}{\sigma_s^2}}^2\\
\leq& \lambda S^2+ \sum_{s=1}^{t}\eta_s^2\left\|\nabla\ell_{s}(\theta_{s})\right\|_{V_{s+1}^{-1}}^2 +\sum_{s=1}^{t}2\eta_s\langle \nabla\widetilde{\ell}_s(\theta_s)-\nabla{\ell}_s(\theta_s),\theta_s-\theta_* \rangle+\\
&\sum_{s=1}^{t}(\frac{1}{\alpha_s}-\frac{\eta_s \cdot\mu'(X_s^T \theta_s)}{1+KS})\lVert \theta_s-\theta_* \rVert_{\frac{X_sX_s^{T}}{\sigma_s^2}}^2.
\end{align*}
This completes the proof.
\end{proof}

\begin{lemma} [gradient loss control] \label{lemma_gradient_loss_control}
If $\sigma_t \geq \max\{\nu_t,\sigma_{\min},\frac{1}{\sqrt{\alpha_t}}\left\|X_t\right\|_{V_{t}^{-1}}\}$, for any $t \geq 1$ and with probability at least $1-\delta$,  we have
\begin{align*}
\sum_{s=1}^{t}\eta_s^2\left\|\nabla\ell_{s}(\theta_{s})\right\|^{2}_{V_{s+1}^{-1}} 
\leq  \max_{1 \leq s \leq t}\left(12\eta_s^2\alpha_s\right)t^{\frac{1-\epsilon}{1+\epsilon}}  
\tau_0^2\log \frac{2T^2}{\delta}
+\sum_{s=1}^{t}2\eta_s^2\alpha_sL^2\frac{w_s^2}{1+w_s^2}\lVert \theta_s-\theta_* \rVert_{\frac{X_sX_s^{T}}{\sigma_s^2}}^2.
\end{align*}
\end{lemma}

\begin{proof}
By computational results in Appendix~\ref{denoising_huber_loss}, we have
\begin{align}
\eta_t^2\left\|\nabla\ell_{t}(\theta_{t})\right\|_{V_{t+1}^{-1}}^{2}
=&\eta_t^2\min\{\left|z_t(\theta_t)\right|,\tau_t\}^2\cdot\left\|\frac{X_t}{\sigma_t}\right\|_{V_{t+1}^{-1}}^{2} \\
\leq& \eta_t^2\min\left\{\left|\frac{\varepsilon_t}{\sigma_t}\right|+\left|\frac{\mu(\langle X_t,\theta_*\rangle)-\mu(\langle X_t,\theta_t\rangle)}{\sigma_t}\right|,\tau_t\right\}^2\cdot\left\|\frac{X_t}{\sigma_t}\right\|_{V_{t+1}^{-1}}^{2}  \\
\leq &\underbrace{2\eta_t^2\min\left\{\left|\frac{\varepsilon_t}{\sigma_t}\right|,\tau_t\right\}^2\left\|\frac{X_t}{\sigma_t}\right\|_{V_{t+1}^{-1}}^{2}}_{TermC.1}+\\
&\underbrace{2\eta_t^2\min\left\{\left|\frac{\mu(\langle X_t,\theta_*\rangle)-\mu(\langle X_t,\theta_t\rangle)}{\sigma_t}\right|,\tau_t\right\}^2\left\|\frac{X_t}{\sigma_t}\right\|_{V_{t+1}^{-1}}^{2}}_{TermC.2},\label{Term C analysis}
\end{align}
where the first inequality comes from triangle inequality. Next we will analyze the above two terms separately. 

For Term C.1, by the Sherman-Morrison formula \citep{sherman1950adjustment}, we claim that
\begin{align}
\left\|\frac{1}{\sqrt{\alpha_t}}\frac{X_t}{\sigma_t}\right\|_{{V}_{t+1}^{-1}}^2
&={\frac{1}{\alpha_t}} \cdot\frac{X_t^T(\lambda I_d+\sum_{s=1}^{t}{\frac{1}{\alpha_s}}\frac{X_sX_s^T}{\sigma_s^2})^{-1}X_t}{\sigma_t^2}\\ 
&= {\frac{1}{\alpha_t}} \cdot\frac{X_t^T({V}_{t}+{\frac{1}{\alpha_t}}\frac{X_tX_t^T}{\sigma_t^2})^{-1}X_t}{\sigma_t^2}  \\
&=\frac{1}{\alpha_t\sigma_t^2}X_t^T\left({V}_{t}^{-1}-\frac{1}{\alpha_t\sigma_t^2} \cdot\frac{{V}_{t}^{-1}X_tX_t^T{V}_{t}^{-1}}{1+w_t^2}\right)X_t \tag{Sherman-Morrison formula}\\
&=w_t^2-\frac{w_t^4}{1+w_t^2}=\frac{w_t^2}{1+w_t^2},\label{S-M_formula}
\end{align}
where $w_t=\left\|\frac{1}{\sqrt{\alpha_t}}\frac{X_t}{\sigma_t}\right\|_{{V}_{t}^{-1}}$. Substituting this into Term~C.1 and summing over $s = 1, \dots, t$, we obtain, 
for any $t \geq 1$ and with probability at least $1-\delta$, that
\begin{align}
&\sum_{s=1}^{t}2\eta_s^2\min\left\{\left|\frac{\varepsilon_s}{\sigma_s}\right|,\tau_s\right\}^2\left\|\frac{X_s}{\sigma_s}\right\|_{V_{s+1}^{-1}}^{2}\\
& =\sum_{s=1}^{t}2\eta_s^2\alpha_s\min\left\{\left|\frac{\varepsilon_s}{\sigma_s}\right|,\tau_s\right\}^2\left\|\frac{1}{\sqrt{\alpha_t}}\frac{X_s}{\sigma_s}\right\|_{V_{s+1}^{-1}}^{2}\\
&\leq \max_{1 \leq s \leq t}\left(2\eta_s^2\alpha_s\right)\sum_{s=1}^{t}\min\left\{\left|\frac{\varepsilon_s}{\sigma_s}\right|,\tau_s\right\}^2 \frac{w_s^2}{1+w_s^2} \tag{because of \eqref{S-M_formula}}\\
& \leq \max_{1 \leq s \leq t}\left(2\eta_s^2\alpha_s\right)t^{\frac{1-\varepsilon}{(1+\varepsilon)}} \cdot \left[  \left( \sqrt{\tau_0^{1-\varepsilon}(\sqrt{2k})^{1+\varepsilon} (\log 3T)^{\frac{1-\varepsilon}{2}}} + \tau_0 \sqrt{2\log \frac{2T^2}{\delta}} \right) \right]^2 \\ 
&\leq \max_{1 \leq s \leq t}\left(2\eta_s^2\alpha_s\right) \cdot
6t^{\frac{1-\varepsilon}{1+\varepsilon}}\tau_0^2\log \frac{2T^2}{\delta}. \label{Term C.1 analysis}
\end{align}
The last inequality holds by applying (24) in Appendix~B.3 of \cite{wang2025heavy}, where we have  set $\tau_0 \geq \frac{\sqrt{2k}(\log 3T)^{\frac{1-\varepsilon}{2(1+\varepsilon)}}}{(\log \frac{2T^2}{\delta})^{\frac{1}{1+\varepsilon}}}$ and $b=\max_{1 \leq s \leq t} \frac{\nu_t}{\sigma_t} \leq 1$.

For Term C.2, a simple computation shows
\begin{align*}
2\eta_t^2\min\left\{\left|\frac{\mu(\langle X_t,\theta_*\rangle)-\mu(\langle X_t,\theta_t\rangle)}{\sigma_t}\right|,\tau_t\right\}^2\left\|\frac{X_t}{\sigma_t}\right\|_{V_{t+1}^{-1}}^{2} 
&\leq 2 \eta_t^2\alpha_tL^2\lVert \theta_t-\theta_* \rVert_{\frac{X_tX_t^{T}}{\sigma_t^2}}^2\frac{w_t^2}{1+w_t^2}.
\end{align*}
Therefore,
\begin{align} \label{Term C.2 Analysis}
&\sum_{s=1}^{t}2\eta_s^2\min\left\{\left|\frac{\mu(\langle X_s,\theta_*\rangle)-\mu(\langle X_s,\theta_s\rangle)}{\sigma_s}\right|,\tau_s\right\}^2\left\|\frac{X_s}{\sigma_s}\right\|_{V_{s+1}^{-1}}^{2}  \\ 
&\leq \sum_{s=1}^{t}2\eta_s^2\alpha_sL^2\frac{w_s^2}{1+w_s^2}\lVert \theta_s-\theta_* \rVert_{\frac{X_sX_s^{T}}{\sigma_s^2}}^2.
\end{align}

Combining \eqref{Term C.1 analysis} and \eqref{Term C.2 Analysis} completes the proof.
\end{proof}

\begin{lemma}[distance control] \label{lemma_distance_control}
If 
\begin{align*}
\sigma_t\geq\max\left\{\nu_t,\sigma_{\min},\sqrt{2\frac{1}{\sqrt{\alpha_t}}\frac{L\beta_{t}\left\|X_t\right\|_{V_{t}^{-1}}^2}{\tau_0t^{\frac{1-\varepsilon}{2(1+\varepsilon)}}}},\frac{1}{\sqrt{\alpha_t}}\left\|X_t\right\|_{V_{t}^{-1}}\right\},
\end{align*}
 then for any $t \geq 1$ and with probability at least $1-2\delta$,  we have
\begin{align*}
&\sum_{s=1}^{t}2\eta_s\langle \nabla\widetilde{\ell}_s(\theta_s)-\nabla{\ell}_s(\theta_s),\theta_s-\theta_* \rangle \mathds{1}_{A_s} \leq \max_{1 \leq s \leq t}\left(19\sqrt{2}\eta_s\sqrt{\alpha_s} \beta_s \right)  \tau_0 t^{\frac{1-\epsilon}{2(1+\epsilon)}} \log(\frac{2T^2}{\delta}).
\end{align*}
\end{lemma}

\begin{proof}
First, observe that 
\begin{align*}
 &2\eta_t\left\langle\nabla\widetilde{\ell}_{t}(\theta_{t})-\nabla\ell_{t}(\theta_{t}),\theta_{t}-\theta_{*}\right\rangle  \mathds{1}_{A_t} \\
&=  \underbrace{2\eta_t\left\langle\nabla\widetilde{\ell}_{t}(\theta_{t})-\nabla\ell_{t}(\theta_{t})+\nabla\ell_{t}(\theta_{*}),\theta_{t}-\theta_{*}\right\rangle \mathds{1}_{A_t} }_{\text{Term~A.1}}
 +\underbrace{2\eta_t\left\langle-\nabla\ell_{t}(\theta_{*}),\theta_{t}-\theta_{*}\right\rangle \mathds{1}_{A_t} }_{\text{Term~A.2}} .
\end{align*}

Recall the expression of extended Huber Loss gradient in Appendix \ref{denoising_huber_loss} and denote $\psi_{\tau_t}(x)$ as standard huber loss gradient, which is given by
\begin{align*}
\psi_{\tau_t}(x)=
\begin{cases}
x & \text{if } |x| \leq \tau_t, \\
\tau_t  & \text{if } x > \tau_t, \\
-\tau_t  & \text{if } x < -\tau_t,
\end{cases}
\end{align*}
we can write
\begin{align} \label{Term A.1 analysis_1}
\text{Term~A.1}&= 2\eta_t\left( -\psi_{\tau_t}(z_t(\theta_t)-z_t(\theta_*))+\psi_{\tau_t}(z_t(\theta_t))-\psi_{\tau_t}(z_t(\theta_*))\right) \frac{\langle X_t,\theta_t-\theta_* \rangle}{\sigma_t}\mathds{1}_{A_t} .
\end{align}
Similar to \eqref{theta_t lies in quadratic area}, when $A_t$ holds, we have
\begin{align} \label{theta_t lies in quadratic area_2}
    |z_t(\theta_t)-z_t(\theta_*)|=|\widetilde{z}_t(\theta_t)| \leq \frac{\tau_t}{2}.
\end{align}
Thus 
\begin{align}
    -\psi_{\tau_t}(z_t(\theta_t)-z_t(\theta_*))=z_t(\theta_*)-z_t(\theta_t).
\end{align}
We continue to discuss two different situations. When $|z_t(\theta_*)| \leq \frac{\tau_t}{2}$,
\begin{align*}
    |z_t(\theta_t)| \leq |z_t(\theta_t)-z_t(\theta_*)|+|z_t(\theta_*)| \leq \tau_t.    
\end{align*}

In this case, we have
\begin{align*}
    -\psi_{\tau_t}(z_t(\theta_t)-z_t(\theta_*))+\psi_{\tau_t}(z_t(\theta_t))-\psi_{\tau_t}(z_t(\theta_*)) = z_t(\theta_*)-z_t(\theta_t)+z_t(\theta_t)-z_t(\theta_*)=0.
\end{align*}
This implies Term~A.1 equals $0$ in this scenario.

Otherwise, if $|z_t(\theta_*)| > \frac{\tau_t}{2}$, it is easy to verify that for any $x,y \in \mathbb{R}$, $|\psi_{\tau_t}(x)-\psi_{\tau_t}(y)| \leq |x-y|$. Hence, by the Cauchy-Schwartz inequality, we can directly give a bound as :
\begin{align*}
\text{Term~A.1}
&\leq 2\eta_t \cdot 2|z_t(\theta_t)-z_t(\theta_*)|\cdot \sqrt{\alpha_t}w_t\beta_t\\
&\leq 2 \eta_t  \tau_0 t^{\frac{1-\epsilon}{2(1+\epsilon)}}\sqrt{1+w_t^2} \sqrt{\alpha_t}\beta_t \tag{because of \eqref{theta_t lies in quadratic area_2}}\\
&\leq 2\sqrt{2}\eta_t\sqrt{\alpha_t} \tau_0t^{\frac{1-\epsilon}{2(1+\epsilon)}}\beta_t. \tag{$w_t \leq 1$}
\end{align*}

Combine these two cases and summing over  $s=1$ to $t$, we have
\begin{align} \label{Term A.1 analysis_2}
&\sum_{s=1}^{t}2\eta_s\left\langle\nabla\widetilde{\ell}_{s}(\theta_{s})-\nabla\ell_{s}(\theta_{s})+\nabla\ell_{s}(\theta_{*}),\theta_{s}-\theta_{*}\right\rangle\mathds{1}_{A_s}  \\
&\leq \sum_{s=1}^{t} \mathds{1}_{\{|z_s(\theta_*)| > \frac{\tau_s}{2}\}}2\sqrt{2}\eta_s\sqrt{\alpha_t}\tau_0 s^{\frac{1-\epsilon}{2(1+\epsilon)}}\beta_s.
\end{align}

Utilizing Eq.(C.12) in \cite{huang2023tackling}, for all $\quad t \geq 1$ and with probability at least $1- \delta$,  we have 
\[
\sum_{s=1}^{t}\mathds{1}\left\{|z_{s}(\theta_{*})|>\frac{\tau_{s}}{2}\right\} \leq \left(4+2\sqrt{2}+\frac{2}{3} \right)\log(\frac{2T^2}{\delta}) \leq \frac{15}{2}\log(\frac{2T^2}{\delta}).
\]
Returning to \eqref{Term A.1 analysis_2},  for all $t \geq 1$ and with probability at least $1-\delta$,  we have
\begin{align} 
 &\sum_{s=1}^{t}2\eta_s\left\langle\nabla\widetilde{\ell}_{s}(\theta_{s})+\nabla\ell_{s}(\theta_{*})-\nabla\ell_{t}(\theta_{s}),\theta_{s}-\theta_{*}\right\rangle  \mathds{1}_{A_s}\\ 
&\leq  \max_{1 \leq s \leq t}\left(15\sqrt{2}\eta_s\sqrt{\alpha_s} \beta_s \right)  \tau_0t^{\frac{1-\epsilon}{2(1+\epsilon)}} \log(\frac{2T^2}{\delta}). \label{Term A.1}
\end{align}

For the second Term A.2, we modified a minor flaw in \cite{wang2025heavy} by a different approach. First we sum over $s=1$ to $t$ to attain
\begin{align} 
\sum_{s=1}^{t}2\eta_s\left\langle-\nabla\ell_{s}(\theta_{*}),\theta_{s}-\theta_{*}\right\rangle \mathds{1}_{A_s} 
&= \sum_{s=1}^{t}2\eta_s \psi_{\tau_s}(z_s(\theta_*))\left\langle\frac{X_s}{\sigma_s},\theta_{s}-\theta_{*}\right\rangle \mathds{1}_{A_s} \\
&\leq \sum_{s=1}^{t}2\eta_s \psi_{\tau_s}(z_s(\theta_*))\sqrt{\alpha_s}w_s\beta_s\mathds{1}_{A_s} \\
&= \max_{1 \leq s \leq t} \left( 2\sqrt{2} \eta_s \sqrt{\alpha_s}\beta_s\right)\tau_0t^{\frac{1-\epsilon}{2(1+\epsilon)}}\sum_{s=1}^{t} M_s \psi_1(\frac{z_s(\theta_*)}{\tau_s}),
\end{align}
where 
\begin{align*}
M_s = \frac{2\eta_s\tau_s\sqrt{\alpha_s}w_s\beta_s\mathds{1}_{A_s} }{\max_{1 \leq s \leq t} \left( 2\sqrt{2} \eta_s \sqrt{\alpha_s}\beta_s\right)\tau_0t^{\frac{1-\epsilon}{2(1+\epsilon)}}} \leq 1.
\end{align*}
By the detailed proof in Appendix~C.6 of \cite{huang2023tackling}, for all $t \geq 1$ and with  probability at least $ 1-\delta$, we have
\begin{align}
   \sum_{s=1}^{t} M_s \psi_1(\frac{z_s(\theta_*)}{\tau_s}) \leq \frac{(\sqrt{2k})^{1+\epsilon}(\log3T)^{\frac{1-\epsilon}{2}}}{\tau_0^{1+\epsilon}}+\log(\frac{2T^2}{\delta}),
\end{align}
where we set $b=\max_{1 \leq s \leq t} \frac{\nu_t}{\sigma_t} \leq 1$. As a result, for all $t \geq 1$ and with probability at least $1-\delta$, we have
\begin{align}
& \sum_{s=1}^{t}2\eta_s\left\langle-\nabla\ell_{s}(\theta_{*}),\theta_{s}-\theta_{*}\right\rangle \mathds{1}_{A_s} \\
& \leq \max_{1 \leq s \leq t} \left( 2\sqrt{2} \eta_s   \sqrt{\alpha_s}\beta_s\right)t^{\frac{1-\epsilon}{2(1+\epsilon)}} \cdot
\left(\frac{(\sqrt{2k})^{1+\epsilon}(\log3T)^{\frac{1-\epsilon}{2}}}{\tau_0^{\epsilon}}+\tau_0\log(\frac{2T^2}{\delta}) \right)\\
& \leq\max_{1 \leq s \leq t}\left( 4\sqrt{2} \eta_s   \sqrt{\alpha_s}\beta_s\right)t^{\frac{1-\epsilon}{2(1+\epsilon)}}\tau_0\log(\frac{2T^2}{\delta}).   \label{Term A.2}
\end{align}

Combining \eqref{Term A.1} and \eqref{Term A.2}, for all $t \geq 1$ and with probability at least $1- 2 \delta$,  we have
\begin{align*}
&\sum_{s=1}^{t}2\eta_s\langle \nabla\widetilde{\ell}_s(\theta_s)-\nabla{\ell}_s(\theta_s),\theta_s-\theta_* \rangle \mathds{1}_{A_s} \\
&\leq \max_{1 \leq s \leq t}\left(15\sqrt{2}\eta_s\sqrt{\alpha_s} \beta_s \right)  \tau_0 t^{\frac{1-\epsilon}{2(1+\epsilon)}} \log(\frac{2T^2}{\delta})
+\max_{1 \leq s \leq t}\left( 4\sqrt{2} \eta_s   \sqrt{\alpha_s}\beta_s\right)t^{\frac{1-\epsilon}{2(1+\epsilon)}}\tau_0\log(\frac{2T^2}{\delta})\\
&= \max_{1 \leq s \leq t}\left(19\sqrt{2}\eta_s\sqrt{\alpha_s} \beta_s \right)  \tau_0 t^{\frac{1-\epsilon}{2(1+\epsilon)}} \log(\frac{2T^2}{\delta}).
\end{align*}
This completes the proof.
\end{proof}

\section{Proof of Theorem ~\ref{thm:confidence_region_bound}}\label{pf_glb_ucb}

\begin{proof}
Based on Lemma ~\ref{lemma_general_upperbound},~\ref{lemma_gradient_loss_control} and ~\ref{lemma_distance_control}, if $A_t$ holds for all $t \geq 1$,
set 
\begin{align*}
\sigma_t=\max\left\{\nu_t,\sigma_{\min},\sqrt{\frac{1}{\sqrt{\alpha_t}}\frac{L\beta_{t}\left\|X_t\right\|_{V_{t}^{-1}}^2}{\tau_0t^{\frac{1-\varepsilon}{2(1+\varepsilon)}}}},\frac{1}{\sqrt{\alpha_t}}\left\|X_t\right\|_{V_{t}^{-1}}\right\}.
\end{align*}

Then, for any $t \geq 1$ and with probability at least $1-3\delta$, we have
\begin{align*}
&\lambda S^2+ \sum_{s=1}^{t}\eta_s^2\left\|\nabla\ell_{s}(\theta_{s})\right\|_{V_{s+1}^{-1}}^2 +\sum_{s=1}^{t}2\eta_s\langle \nabla\widetilde{\ell}_s(\theta_s)-\nabla{\ell}_s(\theta_s),\theta_s-\theta_* \rangle+\\
&\sum_{s=1}^{t}\left(\frac{1}{\alpha_s}-\frac{\eta_s \cdot\mu'(X_s^T \theta_s)}{1+KS}\right)\lVert \theta_s-\theta_* \rVert_{\frac{X_sX_s^{T}}{\sigma_s^2}}^2\\
&\leq \lambda S^2+\max_{1 \leq s \leq t}\left(12\eta_s^2\alpha_s\right)t^{\frac{1-\epsilon}{1+\epsilon}}  
\tau_0^2\log \frac{2T^2}{\delta}
+\sum_{s=1}^{t}2\eta_s^2\alpha_sL^2\frac{w_s^2}{1+w_s^2}\lVert \theta_s-\theta_* \rVert_{\frac{X_sX_s^{T}}{\sigma_s^2}}^2+\\
 & \max_{1 \leq s \leq t}\left(19\sqrt{2}\eta_s\sqrt{\alpha_s} \beta_s \right)  \tau_0 t^{\frac{1-\epsilon}{2(1+\epsilon)}} \log(\frac{2T^2}{\delta})
+  \sum_{s=1}^{t}\left(\frac{1}{\alpha_s}-\frac{\eta_s \cdot\mu'(X_s^T \theta_s)}{1+KS}\right)\lVert \theta_s-\theta_* \rVert_{\frac{X_sX_s^{T}}{\sigma_s^2}}^2.
\end{align*}

First, we analyze the coefficient of $\lVert \theta_s-\theta_* \rVert_{\frac{X_sX_s^{T}}{\sigma_s^2}}^2$. We expect it to be non-postive for further convenience of bounding $\left\|\widehat{\theta}_{t+1}-\theta_{*}\right\|_{V_{t+1}}^{2}$. Subsequently, set
\begin{align*}
\tau_0=\frac{\sqrt{2k}(\log3T)^{\frac{1-\epsilon}{2(1+\epsilon)}}}{(\log\frac{2T^2}{\delta})^{\frac{1}{1+\epsilon}}}, \quad\eta_s=2(1+KS)\mu'(X_s^T\theta_s), \quad
\alpha_s=\frac{1}{\mu'(X_s^T\theta_s)^2}. 
\end{align*}
If $w_s \leq \frac{\mu'(X_s^T\theta_s)}{2\sqrt{2}(1+KS)L}$, we have
\begin{align*}
2\eta_s^2L^2\alpha_s\frac{w_s^2}{1+w_s^2}
+ \frac{1}{\alpha_s}-\frac{\eta_s \cdot\mu'(X_s^T \theta_s)}{1+KS}  
&= 8(1+KS)^2L^2\frac{w_s^2}{1+w_s^2}+\mu'(X_s^T \theta_s)^2-2\mu'(X_s^T \theta_s)^2\\
 & \leq \mu'(X_s^T \theta_s)^2+\mu'(X_s^T \theta_s)^2-2\mu'(X_s^T \theta_s)^2 = 0.
\end{align*}

Therefore, substituting the values to  $\eta_s$ and $\alpha_s$, for any $t \geq 1$ and with probability at least  $1-3\delta$, that
\begin{align*}
&\lambda S^2+ \sum_{s=1}^{t}\eta_s^2\left\|\nabla\ell_{s}(\theta_{s})\right\|_{V_{s+1}^{-1}}^2 +\sum_{s=1}^{t}2\eta_s\langle \nabla\widetilde{\ell}_s(\theta_s)-\nabla{\ell}_s(\theta_s),\theta_s-\theta_* \rangle\mathds{1}_{A_s}+\\
&\sum_{s=1}^{t}\left(\frac{\mu'(X_s^T\theta_s)}{\alpha_s}-\frac{\eta_s \cdot\mu'(X_s^T \theta_s)}{1+KS}\right)\lVert \theta_s-\theta_* \rVert_{\frac{X_sX_s^{T}}{\sigma_s^2}}^2\\
&\leq  \lambda S^2 +48(1+KS)^2t^{\frac{1-\epsilon}{1+\epsilon}}  
\tau_0^2\log \frac{2T^2}{\delta}+ 38\sqrt{2}(1+KS)\max_{1 \leq s \leq t} \left( \beta_s\right)t^{\frac{1-\epsilon}{2(1+\epsilon)}}\tau_0\log \frac{2T^2}{\delta}.
\end{align*}

Then for all $t \geq$1, by choosing
\begin{align*}
    \beta_{t+1}= 54(1+KS)\tau_0t^{\frac{1-\epsilon}{2(1+\epsilon)}}\left( \sqrt{\log(\frac{2T^2}{\delta})}+\frac{1}{15}\right)^2+\sqrt{\lambda}S,
\end{align*}
assuming that $A_t$ holds for all $t \geq 1$,  with probability at least $1-3\delta$, that
\begin{align}
\beta_{t+1}^2 & \geq \lambda S^2+ \sum_{s=1}^{t}\eta_s^2\left\|\nabla\ell_{s}(\theta_{s})\right\|_{V_{s+1}^{-1}}^2 +\sum_{s=1}^{t}2\eta_s\langle \nabla\widetilde{\ell}_s(\theta_s)-\nabla{\ell}_s(\theta_s),\theta_s-\theta_* \rangle\mathds{1}_{A_s}+\\
&\sum_{s=1}^{t}\left(\frac{1}{\alpha_s}-\frac{\eta_s \cdot\mu'(X_s^T \theta_s)}{1+KS}\right)\lVert \theta_s-\theta_* \rVert_{\frac{X_sX_s^{T}}{\sigma_s^2}}^2.\label{beta_t_bound}
\end{align}

For simplicity, let $\mathcal{B}$ denote $\{\forall t \geq 1,\eqref{beta_t_bound} \text{is true} \}$, we obtain that $\mathbb{P}(\mathcal{B})\geq 1-3\delta$. In the following contents, we will use math induction to verify $B \subset \mathcal{A}:=\cap_{t \geq 1}A_t$(therefore we have $\mathbb{P}(\mathcal{A}) \geq \mathbb{P}(B)\geq1-3\delta$). 

Suppose $\mathcal{B}$ is true. For the case $t=1$, $\left\|{\theta}_{1}-\theta_{*}\right\|_{V_{1}}\leq \sqrt{\lambda}S =\beta_1$. When $A_s$ is true for $\forall s \in [t](t \geq1)$, we continue to prove that $A_{t+1}$ also holds. In fact,
\begin{align*}
&\left\|{\theta}_{t+1}-\theta_{*}\right\|_{V_{t+1}}\\
&\leq  \lambda S^2+ \sum_{s=1}^{t}\eta_s^2\left\|\nabla\ell_{s}(\theta_{s})\right\|_{V_{s+1}^{-1}}^2 +\sum_{s=1}^{t}2\eta_s\langle \nabla\widetilde{\ell}_s(\theta_s)-\nabla{\ell}_s(\theta_s),\theta_s-\theta_* \rangle +\\
&\sum_{s=1}^{t}\left(\frac{1}{\alpha_s}-\frac{\eta_s \cdot\mu'(X_s^T \theta_s)}{1+KS}\right)\lVert \theta_s-\theta_* \rVert_{\frac{X_sX_s^{T}}{\sigma_s^2}}^2 \tag{Lemma~\ref{lemma_general_upperbound}}\\
& \leq \beta_{t+1}. \tag{because of \eqref{beta_t_bound}}
\end{align*}
Thus we can conclude $A_t$ is true for all $t \geq 1$ and $\mathbb{P}(\mathcal{A}) \geq \mathbb{P}(B)\geq1-3\delta$.\\

Since $w_t \leq \frac{\mu'(X_t^T\theta_t)}{2\sqrt{2}(1+KS)L} $ implies $w_t \leq 1$, while through simple computation gives
\begin{align*}
\sqrt{2\frac{1}{\sqrt{\alpha_t}}\frac{L\beta_{t+1}\left\|X_t\right\|_{V_{t}^{-1}}^2}{\tau_0t^{\frac{1-\varepsilon}{2(1+\varepsilon)}}}} =\sqrt{\frac{2\mu'(X_t^T\theta_t)L\beta_{t+1}\left\|X_t\right\|_{V_{t}^{-1}}^2}{\tau_0t^{\frac{1-\varepsilon}{2(1+\varepsilon)}}}}    ,
\end{align*}
and
\begin{align*}
    w_t^2 \leq \frac{\mu'(X_t^T\theta_t)}{2\sqrt{2}(1+KS)L} \Leftrightarrow
    \sigma_t \geq \sqrt{2\sqrt{2}(1+KS)L\mu'(X_t^T\theta_t)}\left\|X_t\right\|_{V_{t}^{-1}}.
\end{align*}
Moreover,  the expression of $\beta_t$ shows  that
\begin{align*} \sqrt{\frac{2\mu'(X_t^T\theta_t)L\beta_{t}\left\|X_t\right\|_{V_{t}^{-1}}^2}{\tau_0t^{\frac{1-\varepsilon}{2(1+\varepsilon)}}}}  &\geq \sqrt{\frac{2\mu'(X_t^T\theta_t)L\cdot54(1+KS)\tau_0t^{\frac{1-\varepsilon}{2(1+\varepsilon)}}\cdot0.64\left\|X_t\right\|_{V_{t}^{-1}}^2}{\tau_0t^{\frac{1-\varepsilon}{2(1+\varepsilon)}}}}\\
    &\geq \sqrt{2\sqrt{2}(1+KS)L\mu'(X_t^T\theta_t)}\left\|X_t\right\|_{V_{t}^{-1}}.
\end{align*}

Consequently if we set
\begin{align*}
\sigma_t=\max\left\{\nu_t,\sigma_{\min},\sqrt{\frac{2\mu'(X_t^T\theta_t)L\beta_{t+1}\left\|X_t\right\|_{V_{t}^{-1}}^2}{\tau_0t^{\frac{1-\varepsilon}{2(1+\varepsilon)}}}}\right\},
\end{align*}
then for any $\delta \in (0,1)$, with probability at least $1-3\delta$, for any $t \geq 1$, we have
\begin{align*}
\left\|{\theta}_{t+1}-\theta_{*}\right\|_{V_{t+1}} \leq  54(1+KS)\tau_0t^{\frac{1-\epsilon}{2(1+\epsilon)}}\left( \sqrt{\log(\frac{2T^2}{\delta})}+\frac{1}{15}\right)^2+\sqrt{\lambda}S.
\end{align*}
This completes the proof.
\end{proof}

\section{Proof of Theorem~\ref{thm:regret_bound}}  \label{pf_glb_regret}

\begin{proof}
We first bound the following two expressions by the Cauchy-Schwartz inequality. With probability at least $1-3\delta$, we claim that
\begin{align}\label{arm_bound}
\begin{cases}
\langle X_t^*,\theta_*\rangle \leq \langle X_t^*,\theta_t \rangle+\beta_{t}\left\|X_t^*\right\|_{V_{t}^{-1}}  \leq \langle X_t,\theta_t \rangle+\beta_{t}\left\|X_t\right\|_{V_{t}^{-1}}\quad \text{for any $t \geq 1$},\\
\langle X_t,\theta_*\rangle \geq \langle X_t,\theta_t \rangle-\beta_{t}\left\|X_t\right\|_{V_{t}^{-1}} \quad \text{for any $t \geq 1$}.
\end{cases}
\end{align}
Hence with probability at least $1-3\delta$,
\begin{align}
\langle X_t^*,\theta_*\rangle-\langle X_t,\theta_*\rangle
& \leq 
\langle X_t^*,\theta_t \rangle-\langle X_t,\theta_t \rangle+\beta_{t}\left\|X_t^*\right\|_{V_{t}^{-1}}+\beta_{t}\left\|X_t\right\|_{V_{t}^{-1}}  \\
& \leq2 \beta_{t}\left\|X_t\right\|_{V_{t}^{-1}}.  \label{regret_analysis_1}
\end{align}
The last inequality holds due to arm selection criterion provided in Algorithm~\ref{alg:EHM}. 

Recalling the regret expression given in Section~\ref{section_preliminaries}, we propose the regret upper bound with probability at least $1-3\delta$ that
\begin{align*}
    \mathcal{R}(T)
    &=\sum_{t=1}^{T} [\mu(\langle X_t^*, \theta_*\rangle)-\mu(\langle X_t, \theta_*\rangle)]\\
    &= \underbrace{\sum_{t=1}^{T}\mu'(X_t^T\theta_*)\left(\langle X_t^*, \theta_*\rangle-\langle X_t, \theta_*\rangle\right)}_{\text{Reg(a)}}+\underbrace{\sum_{t=1}^{T}\frac{1}{2}\mu''(\xi_{t,*})\left(\langle X_t^*, \theta_*\rangle-\langle X_t, \theta_*\rangle\right)^2}_{\text{Reg(b)}},
\end{align*}
where $\xi_{t,*}$ lies between $\langle X_t^*, \theta_*\rangle$ and $\langle X_t, \theta_*\rangle$.

Note that
\begin{align}
   \mu'(X_t^T\theta_*) &= \mu'(X_t^T\theta_t )+\mu''(\xi_t)X_t^T(\theta_*-\theta_t) \leq \mu'(X_t^T\theta_t )+KL|X_t^T(\theta_*-\theta_t)|, \label{reg(a)_1}
\end{align}
where $\xi_{t}$ lies between $X_t^T\theta_*$ and $X_t^T\theta_t$.

For Reg(b), by Assumption~\ref{assumption_linkfunc} and Cauchy-Schwartz inequality.
\begin{align*}
    \text{Reg(b)} \leq \sum_{t=1}^{T}\frac{1}{2}KL\beta_t^2\left\|X_t\right\|_{V_{t}^{-1}}^2.
\end{align*}

For Reg(a), we have
\begin{align*}
    \text{Reg(a) }
    &\leq \sum_{t=1}^{T}\left(\mu'(X_t^T\theta_t )+KL|X_t^T(\theta_*-\theta_t)|\right)\left(\langle X_t^*, \theta_*\rangle-\langle X_t, \theta_*\rangle\right) \tag{because of \eqref{reg(a)_1}}\\
    &\leq \sum_{t=1}^{T}\left(\mu'(X_t^T\theta_t )+KL|X_t^T(\theta_*-\theta_t)|\right)2 \beta_{t}\left\|X_t\right\|_{V_{t}^{-1}} \tag{because of \eqref{regret_analysis_1}}\\
    & \leq \underbrace{\sum_{t=1}^{T} 2 \beta_{t} \mu'(X_t^T\theta_t )\left\|X_t\right\|_{V_{t}^{-1}}}_{\text{Reg(a.1)}}+ \underbrace{\sum_{t=1}^{T} 2 KL\beta_t^2\left\|X_t\right\|_{V_{t}^{-1}}^2}_{\text{Reg(a.2)}}.
\end{align*}

We first analysis Reg(a.1):
\begin{align*}
    \text{Reg(a.1)} =\sum_{t=1}^{T} 2 \beta_{t} \left\|\frac{1}{\sqrt{\alpha_t}}\frac{X_t}{\sigma_t}\right\|_{V_{t}^{-1}}
    \leq 2 \beta_{T}\sum_{t=1}^{T} \sigma_tw_t.
\end{align*}
Next, we discuss the situation separately according to the value of $\sigma_t$.

$\textbf{1}^{\circ} \quad  \text{Case} \quad \mathcal{T}_1=\{t \in [T] |\sigma_t \in \{\nu_t,\sigma_{\min}\} \}$:

By setting $\sigma_{\min}  = \frac{1}{\sqrt{T}}$, we have
\begin{align}\label{case_1}
\sum_{t \in \mathcal{T}_1}  \sigma_tw_t \leq \sqrt{\sum_{t=1}^{T} \left(\nu_t^2+\frac{1}{T}\right)}\sqrt{\sum_{t=1}^{T}w_t^2} & \leq \sqrt{2k} \sqrt{\left(\sum_{t=1}^{T}\nu_t^2\right)+1}
= \widetilde{\mathcal{O}}\left(\sqrt{d}\sqrt{\left(\sum_{t=1}^{T}\nu_t^2\right)+1} \right).
\end{align}
by the Cauchy-Schwartz inequality and Elliptical potential lemma (Lemma 7, \cite{wang2025heavy}).

$\textbf{2}^{\circ} \quad  \text{Case} \quad \mathcal{T}_2=\left\{t \in T |\sigma_t = \sqrt{\frac{2\mu'(X_t^T\theta_t)L\beta_{t+1}\left\|X_t\right\|_{V_{t}^{-1}}^2}{\tau_0t^{\frac{1-\varepsilon}{2(1+\varepsilon)}}}}\right\}:$

Note that $\frac{1}{w_t^{2}}$ satisfies
\begin{align*}
\frac{1}{w_t^2} = \frac{\alpha_t\sigma_t^2}{\left\|X_t\right\|_{V_{t}^{-1}}^2} 
\leq \frac{2L\beta_{t+1}}{\mu'(X_t^T\theta_t)\tau_0t^{\frac{1-\epsilon}{2(1+\epsilon)}}} & \leq 2\frac{L}{\kappa}\left(54(1+KS)\left( \sqrt{\log(\frac{2T^2}{\delta})}+\frac{1}{15}\right)^2+\frac{\sqrt{\lambda}S}{\tau_0}  \right)\\
& = \widetilde{\mathcal{O}}(\log T).
\end{align*}
Therefore, we have the estimation
\begin{align} \label{case_2}
\sum_{t \in \mathcal{T}_2}\sigma_tw_t\leq
\sum_{t \in \mathcal{T}_2} \frac{1}{w_t^2}\left\|\frac{1}{\sqrt{\alpha_t}}X_t\right\|_{{V}_{t}^{-1}}w_t^2 &\leq \frac{2kL}{\sqrt{\lambda}}\cdot \widetilde{\mathcal{O}}(\log T)  
= \widetilde{{\mathcal{O}}}(\log^2 T).
\end{align}
Therefore, if we set $\lambda=d$, then
\begin{align*}
\text{Reg(a.1)} \leq 2\beta_T\sum_{t \in \mathcal{T}_1 \cup  \mathcal{T}_2}\sigma_tw_t& \leq   \widetilde{\mathcal{O}}\left(d(1+KS) \sqrt{\left(\sum_{t=1}^{T}\nu_t^2\right)+1}\cdot T^{\frac{1-\epsilon}{2(1+\epsilon)}}\right) +  \widetilde{\mathcal{O}}\left(T^{\frac{1-\epsilon}{2(1+\epsilon)}}\right).
\end{align*}

We go on to analyze Reg(a.2)+Reg(b):
\begin{align*}
\text{Reg(a.2)}+\text{Reg(b)} &\leq  \frac{5}{2}KL\beta_T^2\sum_{t=1}^{T}\left\|X_t\right\|_{V_{t}^{-1}}^2 \\ 
& \leq \frac{5}{2}KL\beta_T^2\sum_{t=1}^{T} \alpha_t\sigma_t^2w_t^2 
\leq 5 \frac{1}{\kappa^2} KL\beta_T^2 k \max_{1 \leq t \leq T}(\sigma_t^2) \leq {\widetilde{\mathcal{O}}}\left( T^{\frac{1-\epsilon}{1+\epsilon}}\right).
\end{align*}
Combining the above bounds, we obtain
\begin{align*}
\mathcal{R}(T)  
&\leq \text{Reg(a.1)}+\text{Reg(a.2)}+\text{Reg(b)}\\
& \leq \widetilde{\mathcal{O}}\left(d(1+KS) \sqrt{\left(\sum_{t=1}^{T}\nu_t^2\right)+1}\cdot T^{\frac{1-\epsilon}{2(1+\epsilon)}} \right)+\widetilde{\mathcal{O}}\left(T^{\frac{1-\epsilon}{2(1+\epsilon)}}\right)
+\widetilde{\mathcal{O}}\left(T^{\frac{1-\epsilon}{1+\epsilon}}\right)\\
&=\widetilde{\mathcal{O}}\left(d(1+KS) T^{\frac{1-\epsilon}{2(1+\epsilon)}} \sqrt{\left(\sum_{t=1}^{T}\nu_t^2\right)+1} \right).
\end{align*}
This completes the proof. 
\end{proof}

\section{Proof of Corollary~\ref{cor:regret_bound}}
\begin{proof}
We only need to replace each $\nu_t$ with $\nu$ and follow the same steps in proof of Theorem \ref{thm:regret_bound}. Thus we have
\begin{align*}
\mathcal{R}(T) &\leq \widetilde{\mathcal{O}}\left( d(1+KS)T^{\frac{1-\epsilon}{2(1+\epsilon)}}\sqrt{\nu^2T+1}\right) 
= \widetilde{\mathcal{O}}\left( d(1+KS)\nu T^{\frac{1}{1+\epsilon}}\right).
\end{align*}
The proof is now completed.
\end{proof}

\section{Proof of Theorem~\ref{thm:PGLB-EHM regret bound}} \label{pf_PA}

\begin{proof}
For any region $A_t^i \triangleq \mathcal{X}_t \cap \mathcal{A}^i \subset \mathcal{X}_t \subset B_d$ (closed unit ball in $\mathbb{R}^d$)

According to Theorem~\ref{thm:regret_bound}, if we only keep exploring in $A_t^i$ from the start, that is, for any $t$, if feasible set becomes $A_t^i$, then with probability at least $1-3\delta$,
\begin{align*}\sum_{t=1}^{T}2\beta_{t}\mu_i'(X_t^T\theta_t)\left\|X_t\right\|_{V_{t}^{-1}}+\frac{1}{2}KL(\beta_t)^2\left\|X_t\right\|_{V_{t}^{-1}}^2
        \leq \widetilde{\mathcal{O}}\left( d(1+KS)\nu T^{\frac{1}{1+\epsilon}}\right).
\end{align*}
For a constant $C>0$, when $T$ is set large enough, with probability at least $1-3\delta$, there are at most $\mathcal{T}_i(C) \leq \widetilde{\mathcal{O}}\left( \frac{1}{C}d(1+KS)\nu T^{\frac{1}{1+\epsilon}}\right)$ times satisfying
\begin{align} \label{arm selection control}
\beta_{t}\mu_i'( \langle X_{t}^i\theta_{t_i}^i \rangle )\left\|X_{t}^i\right\|_{(V_{t_i}^{i})^{-1}}+\frac{1}{2}KL\left(\beta_{t}\left\|X_{t }^i\right\|_{(V_{t_i}^{i})^{-1}} \right)^2 \geq C,
\end{align}
here $\mathcal{T}_i(C) = \{\text{total counts that \eqref{arm selection control} holds}\}$.

Under such circumstance, recall Assumption \ref{assumption_pa_gap} tells us that exists $a>0$ such that for any region $A_t^i$, we have $\mu_t^* \geq \mu_t^i + a$ if $\mu_t^i \neq \mu_t^*$. In other words, if an area $A_i$ contains a local maximal solution that is not global maximal, then it is therefore distinguishable from the global maximal solution by a constant $a >0$. We also notice that for such $i$, with probability at least $1-6\delta$
\begin{align} 
&\big( \langle  X_{t}^i,\theta_{t_i}^i \rangle + \beta_{t_i}\left\| X_{t}^i\right\|_{(V_{t_i}^{i})^{-1}}  \big)  -\mu_i^{-1}(\mu_t^i) \\
& \leq   \big( \langle  X_{t}^i,\theta_{i}^* \rangle + 2\beta_{t_i}\left\| X_{t}^i\right\|_{(V_{t_i}^{i})^{-1}}  \big)  -\mu_i^{-1}(\mu_t^i) \\
& \leq 2\beta_{t_i}\left\| X_{t}^i\right\|_{(V_{t_i}^{i})^{-1}},\label{piecewise_1}
\end{align}
and
\begin{align}
&  \big( \langle  X_{t}^{*},\theta_{t_j}^{j} \rangle + \beta_{t_j}\left\| X_{t}^{j}\right\|_{(V_{t_j}^{j})^{-1}}  \big) -\mu_j^{-1}(\mu_{t}^*) \geq 0 , \label{piecewise_2}
\end{align}
where $j=j(t)$ and $X_t^* \in A_t^{j}$.

Next, we need to verify the selecting effectiveness of the algorithm, which means the algorithm truly chooses the area with global optimal solution under high probability with some exploration cost, and focuses on exploitation in such region at other times. Combining \ref{piecewise_1} and \ref{piecewise_2}, and setting $C=\frac{a\kappa}{3L}$, with probability at least $1-6\delta$, 
\begin{align*}
UCB_{t}^{j} > UCB_{t}^{i}.
\end{align*}
where $UCB_t^i= \max_{x \in \mathcal{A}^i \cap\mathcal{X}_t} \left\{\mu_i \big( \langle x, {\theta}_{t_i}^i \rangle + \beta_{t_i} \|x\|_{(V_{t_i}^{i})^{-1}} \big) \right\}$. 
This indicates through a sufficiently large (forced) exploration with total time cost no more than $\mathcal{T} = m \cdot \max_{1 \leq i \leq m}\mathcal{T}_i(\frac{a\kappa}{3L}) \leq \widetilde{O}\left( \frac{mL}{a\kappa}d(1+KS)\nu T^{\frac{1}{1+\epsilon}}\right)$, the action will focus only on the global optimal areas with probability at least $1-3m\delta$ after $\mathcal{T}$. 

Finally, selecting $\delta=\frac{1}{3mT}$ and using the claim in Theorem ~\ref{thm:regret_bound}, we can quickly give a regret bound in our piecewise generalized linear setting with probability  at least $1-1/T$:
\begin{align*}
    \mathcal{R}(T) &= \sum_{t=1}^T \mu_{t}\left(\langle X_t,\theta^*(X_t) \rangle \right)-\mu_t^* \\
    &\leq 2LS\mathcal{T} + m\cdot \widetilde{\mathcal{O}}\left( d(1+KS)\nu T^{\frac{1}{1+\epsilon}}\right)\\
    & \leq  \widetilde{\mathcal{O}}\left(m(1+\frac{2L^2S}{a\kappa}) d(1+KS)\nu {T}^{\frac{1}{1+\epsilon}}\right).
\end{align*}
This completes the proof.
\end{proof}

\section{Proof of Theorem~\ref{thm:SNB-EHM regret bound}} \label{pf_SNB}

\begin{proof}
For any subregion $A \subset B_d$ that is a refinement during our algorithm, $A=\mathcal{A}\cap B_d$ where $\mathcal{A}=\prod_{i=1}^{d}A_i$ and $A_i(i\in 1,2,..,d)$ is a closed interval with length $d(A)=2^{-h}(h \in \mathbb{N})$. $T(A)$ refers to forced exploration time on the whole area $A$ before thinner refinement happens to it. 

We  first present revised upper confidence bound in such area. Recall that  Appendix~\ref{error_analysis} provides several error estimation lemmas. Denote 
\[
X_A^* = \arg\max_{X\in A} X^T\theta(X)
\]
and let $\theta_A^*=\theta\left(X_A^*\right)$. If $\arg\max_{X\in A} X^T\theta(X)$ is not unique, choose an arbitary element as $X_A^*$.  Consider consistent exploration in area $A$ from the start. By substituting $\theta_A^*$ and $\tau_0(A)$ for $\theta_*$ and $\tau_0$ in Lemma~\ref{lemma_general_upperbound}, we obtain a general upper bound for $||\theta_{t+1}-\theta_A^*||_{V_{t+1}}^2$.

In the revised version of Lemma~\ref{lemma_gradient_loss_control}, analysis for Term C.1 remains the same, while simple computation shows that
\begin{align*}
2\eta_t^2\min\left\{\left|\frac{\mu(\langle X_t,\theta(X_t)\rangle)-\mu(\langle X_t,\theta_t\rangle)}{\sigma_t}\right|,\tau_t\right\}^2\left\|\frac{X_t}{\sigma_t}\right\|_{V_{t+1}^{-1}}^{2} 
&\leq 8 \eta_t^2\alpha_tL^2S^2\frac{w_t^2}{1+w_t^2}.   
\end{align*}
Without loss of generality, we directly set $\sigma_{\min}=1$ for computational convenience. Therefore,
\begin{align*}
\sum_{s=1}^{t}2\eta_s^2\min\left\{\left|\frac{\mu(\langle X_s,\theta(X_s)\rangle)-\mu(\langle X_s,\theta_s\rangle)}{\sigma_s}\right|,\tau_s\right\}^2\left\|\frac{X_s}{\sigma_s}\right\|_{V_{s+1}^{-1}}^{2} 
&\leq  \sum_{s=1}^{t}8 L^2S^2 \eta_s^2\alpha_s\frac{w_s^2}{1+w_s^2} \\
&\leq 16L^2S^2\max_{1 \leq s \leq t}\{ \eta_s^2\alpha_s\}k.
\end{align*}

For Lemma~\ref{lemma_distance_control}, for Term A.1, similarly we have
\begin{align*}
    &\sum_{s=1}^{t}2\eta_s\left\langle\nabla\widetilde{\ell}_{s}(\theta_{s})-\nabla\ell_{s}(\theta_{s})+\nabla\ell_{s}(\theta_{*}),\theta_{s}-\theta_{*}\right\rangle\mathds{1}_{A_s} \\
    & \leq \sum_{s=1}^{t} \mathds{1}_{\{|z_s(\theta_A^*)| > \frac{\tau_s}{2}\}}2\sqrt{2}\eta_s\sqrt{\alpha_t}\tau_0(A) s^{\frac{1-\epsilon}{2(1+\epsilon)}}\beta_s^A,
\end{align*}
where $\tau_0(A)$ will be fixed later. Further analysis indicates that
\begin{align}
  \mathds{1}_{\{|z_s(\theta_A^*)| > \frac{\tau_s}{2}\}}   
&  \leq \mathds{1}_{\{|z_s(\theta(X_s))|+\sqrt{d}L'd(A) > \frac{\tau_s}{2}\}} \leq  \mathds{1}_{\{|z_s(\theta(X_s))|\cdot \tau_s/(\tau_s-2\sqrt{d}L'd(A))> \frac{\tau_s}{2}\}}.
\end{align}
Consequently, if we set \ $\tau_0(A) \geq \max\left\{2\sqrt{2d}L'd(A),\frac{\sqrt{2k}(\log 3T)^{\frac{1-\varepsilon}{2(1+\varepsilon)}}}{(\log \frac{2T^2}{\delta})^{\frac{1}{1+\varepsilon}}}\right\}$ and $\sigma_s \geq 2\nu_s$, then
\begin{align*}
&\sum_{s=1}^{t}2\eta_s\left\langle\nabla\widetilde{\ell}_{s}(\theta_{s})-\nabla\ell_{s}(\theta_{s})+\nabla\ell_{s}(\theta_{*}),\theta_{s}-\theta_{*}\right\rangle\mathds{1}_{A_s}  \\ 
& \leq \max_{1 \leq s \leq t}\left(15\sqrt{2}\eta_s\sqrt{\alpha_s} \beta_s^A \right)  \tau_0(A)t^{\frac{1-\epsilon}{2(1+\epsilon)}} \log(\frac{2T^2}{\delta}).
\end{align*}

As for revised Term A.2, following similar steps from Proof of Lemma~\ref{lemma_distance_control}, we have
\begin{align*}
    &\sum_{s=1}^{t}2\eta_s\left\langle-\nabla\ell_{s}(\theta_A^*),\theta_{s}-\theta_A^*\right\rangle \mathds{1}_{A_s}  \\
    & \leq \sum_{s=1}^{t}2\eta_s\left\langle-\nabla\ell_{s}(\theta(X_s)),\theta_{s}-\theta_A^*\right\rangle \mathds{1}_{A_s}+\sum_{s=1}^{t}2\eta_s\left\langle\nabla\ell_{s}(\theta(X_s))-\nabla\ell_{s}(\theta_A^*),\theta_{s}-\theta_A^*\right\rangle \mathds{1}_{A_s}  \\ 
    &\leq \max_{1 \leq s \leq t}\left(19\sqrt{2}\eta_s\sqrt{\alpha_s} \beta_s^A \right)  \tau_0(A) t^{\frac{1-\epsilon}{2(1+\epsilon)}} \log(\frac{2T^2}{\delta})+ \sum_{s=1}^{t}2\eta_s\langle \frac{X_s}{\sigma_s},\theta_s-\theta_A^* \rangle \cdot |z_s(\theta(X_s))-z_s(\theta_A^*)|\\
    & \leq \max_{1 \leq s \leq t}\left(19\sqrt{2}\eta_s\sqrt{\alpha_s} \beta_s^A \right)  \tau_0(A) t^{\frac{1-\epsilon}{2(1+\epsilon)}} \log(\frac{2T^2}{\delta})+  \sum_{s=1}^{t}\eta_s^2\lVert \theta_s-\theta_A^* \rVert_{\frac{X_sX_s^{T}}{\sigma_s^2}}^2 +\sum_{s=1}^{t}\lVert \theta(X_s)-\theta_A^* \rVert_{\frac{X_sX_s^{T}}{\sigma_s^2}}^2\\
    & \leq   \max_{1 \leq s \leq t}\left(19\sqrt{2}\eta_s\sqrt{\alpha_s} \beta_s^A \right)  \tau_0(A) t^{\frac{1-\epsilon}{2(1+\epsilon)}} \log(\frac{2T^2}{\delta}) + \sum_{s=1}^{t}\eta_s^2\lVert \theta_s-\theta_A^* \rVert_{\frac{X_sX_s^{T}}{\sigma_s^2}}^2 +d\cdot T(A) L'^2 d(A)^2.
\end{align*}
The last inequality holds due to Assumption~\ref{assumption_SNB Lip continuity}. 

Recall in nonlinear bandit we reset 
\begin{align*}
  \sigma_t=\max\left\{2\nu_t,1,\sqrt{2\frac{1}{\sqrt{\alpha_t}}\frac{L\beta_{t}^A\left\|X_t\right\|_{V_{t}^{-1}}^2}{\tau_0(A)t^{\frac{1-\varepsilon}{2(1+\varepsilon)}}}},\frac{1}{\sqrt{\alpha_t}}\left\|X_t\right\|_{V_{t}^{-1}}\right\}  
\end{align*}
and
\begin{align*}
  \tau_0(A) = \max\left\{2\sqrt{2d}L'd(A),\frac{\sqrt{2k}(\log 3T)^{\frac{1-\varepsilon}{2(1+\varepsilon)}}}{(\log \frac{2T^2}{\delta})^{\frac{1}{1+\varepsilon}}}\right\}   ,\tau_t=\tau_0(A) \frac{\sqrt{1+w_t^2}}{w_t}t^{\frac{1-\varepsilon}{2(1+\varepsilon)}}.
\end{align*}

Combining the above bounds yields a final closed-form upper confidence bound in the area $A$ (within time limit $T(A)$) with probability at least  $1-3\delta$:
\begin{align*}
&\left\|{\theta}_{t+1}-\theta_A^*\right\|_{V_{t+1}}^{2}\\
& \leq   \lambda S^2+ \sum_{s=1}^{t}\eta_s^2\left\|\nabla\ell_{s}(\theta_{s})\right\|_{V_{s+1}^{-1}}^2 +\sum_{s=1}^{t}2\eta_s\langle \nabla\widetilde{\ell}_s(\theta_s)-\nabla{\ell}_s(\theta_s),\theta_s-\theta_A^* \rangle+\sum_{s=1}^{t}C_s\lVert \theta_s-\theta_A^* \rVert_{\frac{X_sX_s^{T}}{\sigma_s^2}}^2  \\
& \leq   \lambda S^2+\max_{1 \leq s \leq t}\left(2\eta_s^2\alpha_s\right) \cdot
6t^{\frac{1-\varepsilon}{1+\varepsilon}}\tau_0^2(A)\log \frac{2T^2}{\delta}
+16L^2S^2\max_{1 \leq s \leq t}\{ \eta_s^2\alpha_s\}k+\\
& \max_{1 \leq s \leq t}\left(19\sqrt{2}\eta_s\sqrt{\alpha_s} \beta_s^A \right)  \tau_0(A) t^{\frac{1-\epsilon}{2(1+\epsilon)}} \log(\frac{2T^2}{\delta})  +d\cdot T(A) L'^2 d(A)^2 +\\
&\sum_{s=1}^{t}(\eta_s^2+C_s)\lVert \theta_s-\theta_A^* \rVert_{\frac{X_sX_s^{T}}{\sigma_s^2}}^2,
\end{align*}
where $C_s=\frac{1}{\alpha_s}-\frac{\eta_s \cdot\mu'(X_s^T \theta_s)}{1+KS}$. Therefore, if we set 
\begin{align*}
    \alpha_s=\frac{4(1+KS)^2}{\mu'(X_s^T \theta_s)^2},\eta_s=\frac{\mu'(X_s^T \theta_s)}{2(1+KS)},
\end{align*}
then 
\begin{align*}
  \left\|{\theta}_{t+1}-\theta_A^*\right\|_{V_{t+1}}^{2}
  \leq &\lambda S^2 + 12\tau_0^2(A)t^{\frac{1-\varepsilon}{1+\varepsilon}}\log \frac{2T^2}{\delta}+16dL^2S^2\log(1+\frac{L^2T}{\lambda d})+\\
  &19\sqrt{2}\tau_0(A) t^{\frac{1-\epsilon}{2(1+\epsilon)}} \log(\frac{2T^2}{\delta})+U(A),
\end{align*}
where $U(A)=d\cdot T(A) L'^2 d(A)^2$. Hence, by choosing 
\begin{align}
    \beta_{t+1}^A= 27\tau_0(A)t^{\frac{1-\epsilon}{2(1+\epsilon)}}\left( \sqrt{\log(\frac{2T^2}{\delta})}+\frac{1}{12} \right)^2+\sqrt{\lambda} S +4LS\sqrt{d\log(1+\frac{L^2T}{\lambda d})}+\sqrt{U(A)},
\end{align}
and utilizing similar proof skill in Appendix~\ref{pf_glb_ucb}, we have, for any $t \geq 1$ and with probability at least  $1-3\delta$, that 
\begin{align*}
    \left\|{\theta}_{t+1}-\theta_A^*\right\|_{V_{t+1}}^{2} \leq \beta_{t+1}^A.
\end{align*}

At the same time,
\begin{align} \label{sigma_t_1}
\sigma_t\geq\sqrt{2\frac{1}{\sqrt{\alpha_t}}\frac{L\beta_{t+1}^A\left\|X_t\right\|_{V_{t}^{-1}}^2}{\tau_0(A)t^{\frac{1-\varepsilon}{2(1+\varepsilon)}}}} 
\Leftrightarrow 
\sigma_t^2 \geq \frac{\mu'(X_t^T\theta_t)L\beta_{t+1}^A\left\|X_t\right\|_{V_{t}^{-1}}^2}{(1+KS)\tau_0(A)t^{\frac{1-\epsilon}{2(1+\epsilon)}}}
\end{align}
and 
\begin{align}\label{sigma_t_2}
w_t^2 \leq \frac{\mu'(X_t^T\theta_t)\tau_0(A)t^{\frac{1-\varepsilon}{2(1+\varepsilon)}}}{(1+KS)L\beta_{t+1}^A} 
\Rightarrow w_t^2 \leq \frac{1}{16} \Rightarrow w_t^2 \leq 1.
\end{align}

Combining \eqref{sigma_t_1} and \eqref{sigma_t_2}  indicates that
\begin{align*}
    \sigma_t = \max\left\{2\nu_t,1,\sqrt{\frac{\mu'(X_t^T\theta_t)L\beta_{t+1}^A\left\|X_t\right\|_{V_{t}^{-1}}^2}{(1+KS)\tau_0(A)t^{\frac{1-\varepsilon}{2(1+\varepsilon)}}}}\right\}. 
\end{align*}

We also need to analyze the gap between $\mu^{-1}(\mu_A^*)\triangleq\max_{X \in A}X^T\theta(X)$ and $UCB_t^A\triangleq\max_{X \in A} X^T\theta_t+\beta_t^A \left\|X\right\|_{V_{t}^{-1}}$, inspired from Appendix~\ref{pf_glb_regret}:
\begin{align}
   UCB_t^A-\mu^{-1}(\mu_A^*)
&\leq  X_t^T\theta_t+\beta_t^A \left\|X_t\right\|_{V_{t}^{-1}}-X_t^T\theta_A^*+(X_t-X_A^*)^T\theta_A^*     \\
&\leq 2\beta_t^A \left\|X_t\right\|_{V_{t}^{-1}}+\sqrt{d}Sd(A).
\end{align}

Meanwhile the cumulative error
\begin{align*}
\sum_{t=1}^{T(A)} \beta_t^A \left\|X_t\right\|_{V_{t}^{-1}}  
& \leq \beta_{T(A)}^A\sum_{t=1}^{T(A)} \sigma_tw_t.
\end{align*}

$\textbf{1}^{\circ} \quad  \text{Case} \quad \mathcal{T}_1=\{t \in [T(A)] |\sigma_t \in \{2\nu_t,1\} \}$:

We have
\begin{align}\label{case_1}
\sum_{t \in \mathcal{T}_1}  \sigma_tw_t \leq \sqrt{\sum_{t=1}^{T(A)} \left(4\nu_t^2+1\right)}\sqrt{\sum_{t=1}^{T(A)}w_t^2} & \leq \sqrt{2k} \sqrt{\sum_{t=1}^{T(A)}\left(4\nu_t^2+1\right)}\\
& = \sqrt{2k\left(1+4\frac{\sum_{t=1}^{T(A)}\nu_t^2}{T(A)}\right)}\cdot \sqrt{T(A)}.
\end{align}

$\textbf{2}^{\circ} \quad  \text{Case} \quad \mathcal{T}_2=\left\{t \in [T(A)] |\sigma_t = \sqrt{\frac{\mu'(X_t^T\theta_t)L\beta_{t+1}^A\left\|X_t\right\|_{V_{t}^{-1}}^2}{(1+KS)\tau_0(A)t^{\frac{1-\varepsilon}{2(1+\varepsilon)}}}}\right\}$:

Notice that $\frac{1}{w_t^{2}}$ satisfies
\begin{align*}
\frac{1}{w_t^2} = \frac{\alpha_t\sigma_t^2}{\left\|X_t\right\|_{V_{t}^{-1}}^2} 
\leq &\frac{4(1+KS)L\beta_{t+1}^A}{\mu'(X_t^T\theta_t)\tau_0(A)t^{\frac{1-\epsilon}{2(1+\epsilon)}}} \\
 \leq &\frac{4(1+KS)L}{\mu'(X_t^T\theta_t)} \cdot\\
&\left(27\left( \sqrt{\log(\frac{2T^2}{\delta})}+\frac{1}{12}\right)^2+\frac{\sqrt{\lambda}S+4LS\sqrt{d\log(1+\frac{L^2T}{\lambda d})}+\sqrt{U(A)}}{\tau_0(A)}  \right).
\end{align*}
As a result,
\begin{align*}
  \sum_{t \in \mathcal{T}_2}\sigma_tw_t&\leq
\sum_{t \in \mathcal{T}_2} \frac{1}{w_t^2}\left\|\frac{1}{\sqrt{\alpha_t}}X_t\right\|_{{V}_{t}^{-1}}w_t^2 \leq \frac{4(1+KS)kL}{\sqrt{\lambda}}\cdot V(A),
\end{align*}
where 
\begin{align*}
    V(A)=27\left( \sqrt{\log(\frac{2T^2}{\delta})}+\frac{1}{12}\right)^2+\frac{\sqrt{\lambda}S+4LS\sqrt{d\log(1+\frac{L^2T}{\lambda d})}+\sqrt{U(A)}}{\tau_0(A)}. 
\end{align*}

Summing up both two cases tells
\begin{align*}
\sum_{t=1}^{T(A)} \beta_t^A \left\|X_t\right\|_{V_{t}^{-1}}  
& \leq \beta_{T(A)}^A\sum_{t=1}^{T(A)} \sigma_tw_t\\
& \leq \beta_{T(A)}^A \left( \sqrt{2k\left(1+4\frac{\sum_{t=1}^{T(A)}\nu_t^2}{T(A)}\right)}\cdot \sqrt{T(A)}+\frac{4(1+KS)kL}{\sqrt{\lambda}}\cdot V(A)\right).
\end{align*}

Next we  follow a special proof technique to derive an upper bound of the expected regret. For any local optimal feasible area $A$ (which contains no global maximum) in our algorithm, define 
$\Delta_A:= \mu^{-1}(\mu^*)-\mu^{-1}(\mu_A^*)>0$ 
and 
$\delta_A :=\mu^{-1}(\mu^*)-\inf_{X \in A}X^T\theta(X) \geq \Delta_A$. 
If 
\begin{align} \label{ucb_criterion}
2\beta_t^A \left\|X_t\right\|_{V_{t}^{-1}}+\sqrt{d}Sd(A)
\leq \Delta_A/2,
\end{align}
then with probability at least $1-6\delta$, we have
\begin{align*}
UCB_t^A &\leq  \mu^{-1}(\mu_A^*) +   2\beta_t^A\left\|X_t\right\|_{V_{t}^{-1}}+\sqrt{d}Sd(A) < \mu^{-1}(\mu_A^*)+ \Delta_A \leq \mu^{-1}(\mu^*) \leq UCB_t^*,
\end{align*}
here $UCB_t^A=\max_{x \in A \cap B_d} \left\{ \langle x, {\theta}_{t(A)}^A \rangle + \beta_{t(A)} \|x\|_{(V_{t(A)}^{A})^{-1}} \right\}$ and $UCB_t^*=UCB_t^{A^*}$ where $X^* \in A^*$. The last equality holds due to arm selection criterion in $A^*$. In this sense, the algorithm will stop further exploration and refinement in area $A$ ever since, with high probability. 

To make 
\begin{align*}
2\beta_t^{A'} \left\|X_t\right\|_{V_{t}^{-1}}+\sqrt{d}Sd(A')
\leq \Delta_{A'}/2,    
\end{align*}
finally stand for some refined subregion $A'$ of $A$, we only require the following two inequalities:
\begin{align}
   &\beta_t^{A'} \left\|X_t\right\|_{V_{t}^{-1}} \leq \frac{\Delta_A}{6} 
   \leq \frac{\Delta_{A'}}{6},\label{criterion_1}\\
   &d(A') \leq  \frac{\Delta_A}{6\sqrt{d}S}. \label{criterion_2}
\end{align}

Note that \eqref{criterion_1} is satisfied as long as
\begin{align}\label{criterion_1.1}
    \frac{T(A')\Delta_A}{6} \geq \beta_{T(A')}^{A'} \left( \sqrt{2k\left(1+4\frac{\sum_{t=1}^{T(A')}\nu_t^2}{T(A')}\right)}\cdot \sqrt{T(A')}+\frac{4(1+KS)kL}{\sqrt{\lambda}}\cdot V(A')\right).
\end{align}
Moreover, as long as
\begin{align}\label{region_A'}
   T(A') \geq \frac{6}{\Delta_A}\beta_{T+1}^{A'}\left( \sqrt{2k\left(1+4\frac{\sum_{t=1}^{T(A')}\nu_t^2}{T(A')}\right)}\cdot \sqrt{T(A')} +\frac{4(1+KS)kL}{\sqrt{\lambda}}\cdot V(A')\right).
\end{align}

Finally, we consider those areas $A$($ d(A)=2^{-h}, h=1,2,..$) that are the son of some global optimal region. It is easy to verify that
 $ \delta_A \leq 2\sqrt{d}(S+L')d(A)$ and $\Delta_A \leq \sqrt{d}(S+L')d(A) $ Later, we will use math induction to analyze the upper bound of expected time spent on them.
 
At the beginning of algorithm, $d$-dimensional cube $[-1,1]^d$ is divided into $2^d$ parts by bisection on each dimension, and intersecting them with $B_d$ produces $2^d$ subregions of $B_d$ (some of them may be empty set). For the same reason, every refinement on the global optimal region creates $2^d$ subregions, among which at least exists a globally optimal region, and other $2^{d}-1$ regions that may not be global optimal(even if so, $X^*$ is on the vertices of these regions). For these $2^d-1$ regions, randomly select one $A$ for analysis. Its self-refinement will go on to produce $2^d$ smaller subsubregions. 

It is not difficult to verify that Lipschitz continuity on $\theta(\cdot)$ and assumption of the above refined area imply a quick upper bound for all these smaller subsubregions $A''$: 
\begin{align}\label{prop_delta_1}
    \delta_{A''} \leq \delta_A\leq  2\sqrt{d}(S+L')d(A)= 4\sqrt{d}(S+L')d(A'').
\end{align}
Moreover recall that $X^* \in \overline{A^C}$, at least $2^d-1$ sons of $A$ (namely $\mathbf{A}$) satisfy 
\begin{align}\label{prop_delta_2}
\Delta_{\mathbf{A}} \geq \frac{C_1}{L} d(X^*,\mathbf{A})^{C_2}\geq \frac{C_1}{L} d(\mathbf{A})^{C_2}.
\end{align}
\eqref{prop_delta_1} and \eqref{prop_delta_2} are trivial topological properties to certificate in $\mathbb{R}^d$.  After that, there may exist one remaining subregion $\mathbf{B}$ with $\Delta_{\mathbf{B}}< \frac{C_1}{L} d(\mathbf{B})^{C_2}$, for which it is is hard to give a precise lower bound, but we can still follow the similar analysis to \eqref{prop_delta_1} and \eqref{prop_delta_2} to divide it into parts. In general, we provide a discussion in two situations.

$\textbf{1}^{\circ} X^* \in \partial \mathbf{B}$. 

In this case, $\mathbf{B}$ is the global optimal area and $\Delta_{\mathbf{B}}=0$. If we refine it into $2^d$ parts, we will immdiately get $2^d-1$ areas satisfy property \eqref{prop_delta_1} and \eqref{prop_delta_2}, and another 1 area $\mathbf{B'}$ with  property \eqref{prop_delta_1} and $X^* \in \partial \mathbf{B'}$. The similar analysis goes on by induction.

$\textbf{2}^{\circ} X^* \not\in \partial \mathbf{B} \subset \overline{\mathbf{B}}$. 

This implies $\Delta_\mathbf{B} > 0 $. Again, if we refine it, $2^d-1$ areas satisfy \eqref{prop_delta_1} and \eqref{prop_delta_2}, while one remaining subregion $\mathbf{B'}$ with $\Delta_\mathbf{B'}\geq \Delta_\mathbf{B}>0 $ may not meet \eqref{prop_delta_2}. Eventually, there exists some time $\mathbf{t}$ when such $d(\mathbf{B'})$ becomes small enough so that $\Delta_\mathbf{B'}\geq \Delta_\mathbf{B} \geq \frac{C_1}{L}d(\mathbf{B'})^{C_2}$. We can incorporate  into our regret analysis, which is provided in the following paragraphs, by controlling its total exploration time then on.

Finally, we need to settle down some important parameters to make our regret analysis work. Let
\begin{align*}
    \nu_t \leq \nu , T(A')=\frac{1}{d(A')^2} \leq T,\delta=\frac{1}{6T},
\end{align*}
where
\begin{align*}
    \hfill \beta_{t+1}^{A'} &= 27\tau_0(A')t^{\frac{1-\varepsilon}{2(1+\varepsilon)}}\left( \sqrt{\log\left(\frac{2T^2}{\delta}\right)}+\frac{1}{12} \right)^2+\sqrt{\lambda} S +4LS\sqrt{d\log(1+\frac{L^2T}{\lambda d })}
    +\sqrt{U}\\
    &\leq  \widetilde{\mathcal{O}}\left(M_0T^{\frac{1-\epsilon}{2(1+\epsilon)}}\right),\\
    \hfill U(A') &= U=dL'^2, &&\\
    \hfill V(A') &= 27\left( \sqrt{\log\left(\frac{2T^2}{\delta}\right)}+\frac{1}{12} \right)^2+\frac{\sqrt{\lambda}S+4LS\sqrt{d\log(1+\frac{L^2T}{\lambda d})}+\sqrt{U(A')}}{\tau_0(A')}  \leq  {\mathcal{O}}(\log T),
\end{align*}
where $M_0 = \max\left\{ 2 \sqrt{2d}L', \frac{\sqrt{2k}(\log 3T)^{\frac{1-\varepsilon}{2(1+\varepsilon)}}}{(\log \frac{2T^2}{\delta})^{\frac{1}{1+\varepsilon}}}\right\} $. Moreover, for refinements on sons of region $A$ which satisfy \eqref{prop_delta_2}, \eqref{region_A'} and \eqref{criterion_2} are true  so long as
\begin{align*}
    1/d(A')^2&\geq \frac{2^{2C_2}L^2}{C_1^2d(A)^{2C_2}} \cdot \left(144k(1+4\nu^2)(\beta_{T+1}^{A'})^2+  \frac{48(1+KS)k}{\sqrt{\lambda}}\cdot C_1d(A)^{C_2}\cdot V(A')\right)\\
    &=\frac{1} {d(A)^{2C_2}}\widetilde{\mathcal{O}}(M_1^2 T^{\frac{1-\varepsilon}{1+\varepsilon}}),\\
    1/d(A') &\geq \frac{6\sqrt{d}   LS}{C_1d(A)^{C_2}},
\end{align*}
which is equivalently expressed as
\begin{align} \label{k_selection}
    2^{k} \geq d(A)^{1-C_2}\widetilde{\mathcal{O}}(M_1T^{\frac{1-\epsilon}{2(1+\epsilon)}}),
\end{align}
by taking the square root of both sides in the first inequality if $1/d(A')=2^k/d(A)$ and $M_1 = \frac{2^{C_2}L\sqrt{12d(1+4\nu^2)}M_0}{C_1}$. We can choose $k_1(A)\in \mathbb{N}^+$ with $2^{k_1(A)}=d(A)^{1-C_2}\widetilde{\mathcal{O}}(M_1T^{\frac{1-\varepsilon}{2(1+\varepsilon)}})$ to fulfill \eqref{k_selection}.  

In this sense, consider expected total time cost on $A$ (including time cost on its refined sub-regions, except for the exploration time on special regions like $\mathbf{B}$ needs), namely $\mathbb{E}[T_{total}(A)]$. We have 
\begin{align*}
  \mathbb{E}[T_{total}(A)] \leq \sum_{k=0}^{k_1(A)}\left(\frac{2^k}{d(A)}\right)^2\cdot2^{dk}+ 6\delta T\leq d(A)^{(d+2)(1-C_2)-2}\widetilde{\mathcal{O}}\left(\frac{4}{3}M_1^{d+2}T^{\frac{(d+2)(1-\varepsilon)}{2(1+\varepsilon)}} \right). 
\end{align*}
Here we require $\varepsilon >\frac{d}{d+4} $ to make $\frac{(d+2)(1-\varepsilon)}{1+\varepsilon} < 1$. 

Meanwhile, we should remember to take forced exploring times and errors generated by areas like $\mathbf{B}$ into account. By our previous analysis, every son of global optimal region may at most contribute 1 similar region. So, for $h=1,2,..$, there are at most $(h-1)2^d$ of refined regions like $\mathbf{B}$ with $d(\mathbf{B})=2^{-h}$ that satisfy property \eqref{prop_delta_1} but fail to satisfy \eqref{prop_delta_2}. Let $T_h$ denote the time spent on area $A$ such that $  \delta_A \leq 2\sqrt{d}(S+L')L \cdot2^{-(h-1)}$. In addition, if $A$ can be controlled by $  \delta_A \leq 2\sqrt{d}(S+L')L \cdot2^{-h_A}$ yet fails to be controlled by $  \delta_A \leq 2\sqrt{d}(S+L')L \cdot2^{-(h_A+1)}$ by our knowledge, then it will only be counted and discussed in the form $ \delta_A \leq 2\sqrt{d}(2S+L')L \cdot2^{-h_A}$ to make regret reach its upper bound. Therefore,
\begin{align*}
    \mathbb{E}[T_h] 
    \leq h \cdot 2^d \cdot \mathbb{E}[T_{total}(A)] 
    \leq  h \cdot 2^{(h-1)[(d+2)(C_2-1)+2]}\widetilde{\mathcal{O}}\left(M_2T^{\frac{(d+2)(1-\varepsilon)}{2(1+\varepsilon)}} \right), 
\end{align*}
where $M_2=\frac{2^{d+2}}{3}M_1^{d+2}$.The last inequality holds because $k \leq 2^k$. Here factor $2^d$ and $h$ fully account for the refinement time on global optimal area and potential time waste spent on those special $\mathbf{B}$'s that may be hard to  distinguish. Hence, 
\begin{align*}
    \sum_{h=1}^{k} \mathbb{E}[T_h] 
    &\leq   k\cdot2^{(k-1)[(d+2)(C_2-1)+2]} \cdot \widetilde{\mathcal{O}}(\frac{4}{3}M_2T^{\frac{(d+2)(1-\varepsilon)}{2(1+\varepsilon)}}) =\varphi(k). 
\end{align*}

Hence we can choose $k_2 \in \mathbb{N}^+$ that satisfies $\varphi(k_2) \leq T $ and $k_2\cdot2^{(k_2-1)[(d+2)(C_2-1)+2]} = \frac{T^{1-\frac{(d+2)(1-\epsilon)}{2(1+\varepsilon)}}}{\widetilde{\mathcal{O}}(\frac{4}{3}M_2)} $. Denote $\gamma=(d+2)(C_2-1)+1 \geq 1$, we obtain
\begin{align}
    &k_2\cdot2^{(k_2-1)(\gamma+1)} = \frac{T^{1-\frac{(d+2)(1-\epsilon)}{2(1+\varepsilon)}}}{\widetilde{\mathcal{O}}(\frac{4}{3}M_2)}    
    \Rightarrow 2^{(k_2-1)(\gamma+2)} \geq \frac{T^{1-\frac{(d+2)(1-\epsilon)}{2(1+\varepsilon)}}}{\widetilde{\mathcal{O}}(\frac{2}{3}M_2)} \Leftrightarrow 2^{k_2-1} \geq \frac{T^{[1-\frac{(d+2)(1-\epsilon)}{2(1+\varepsilon)}]\frac{1}{\gamma+2}}}{\widetilde{\mathcal{O}}(M_2^{\frac{1}{\gamma+2}})}.\label{k'_prop}
\end{align}

The first inequality holds because $k_2 \leq 2^{k_2}$. Using \eqref{k'_prop}, and consequently,
\begin{align}
    \mathbb{E}[\mathcal{R}(T)] &\leq 2\sqrt{d}(S+L')L\left(\sum_{h=1}^{k_2} 2^{-h}\mathbb{E}[T_h]+ 2^{-k_2}\left(T-\sum_{h=1}^{k_2} \mathbb{E}[T_h]\right)\right)\\
    &\leq  \sqrt{d}(S+L')L \left(\frac{\varphi(k_2) }{2^{k_2-1}}  +\frac{T}{2^{k_2-1}} \right) \\
    & \leq 2\sqrt{d}(S+L')L\frac{T}{2^{k_2-1}} \\
    & = \widetilde{O}\left(MT^{\alpha}\right),
\end{align}
where  $M=2\sqrt{d}(S+L')LM_2^{\frac{1}{\gamma+2}}$,$\alpha = \frac{\gamma+1}{\gamma+2}+\frac{(d+2)(1-\varepsilon)}{2(\gamma+2)(1+\varepsilon)} $ ,$\gamma=(d+2)(C_2-1)+1 $ and $\varepsilon >\frac{d}{d+4}$.
\end{proof}
\end{document}